\documentclass[twoside]{article}

\usepackage[accepted]{aistats2020}
\usepackage[round]{natbib}
\usepackage{cases} 
\usepackage{comment}

\newcommand{\trunc}{\textrm{approx}}
\newcommand{\full}{\textrm{expanded}}
\newcommand{\supp}{\textrm{supp}}
\newcommand {\abs}[1]{\left\vert\, #1 \,\right\vert}
\newcommand{\Jcal}{\mathcal{J}}
\newcommand{\Lcal}{\mathcal{L}}

\usepackage[utf8]{inputenc}

\usepackage{graphicx}
\usepackage{algorithm}
\usepackage{algorithmic}
\usepackage{amsmath}
\usepackage{mathtools}
\usepackage{amssymb}
\usepackage{amsthm}
\usepackage{stmaryrd}
\usepackage{xcolor}
\usepackage{mathtools}
\usepackage{mathrsfs}
\usepackage{dsfont}
\usepackage{caption}
\usepackage{subcaption}
\usepackage{enumitem}
\usepackage{titlesec}
\usepackage{hyperref}
\usepackage{blkarray}
\usepackage{wrapfig}
\SetSymbolFont{stmry}{bold}{U}{stmry}{m}{n}
\usepackage[symbol]{footmisc}  

\usepackage[colorinlistoftodos,prependcaption,
textsize=tiny]{todonotes}
\makeatletter
\renewcommand*\env@matrix[1][*\c@MaxMatrixCols c]{%
  \hskip -\arraycolsep
  \let\@ifnextchar\new@ifnextchar
  \array{#1}}
\makeatother

\theoremstyle{plain}\newtheorem{proposition}{Proposition}[section]
\theoremstyle{plain}\newtheorem{theorem}{Theorem}[section]
\theoremstyle{plain}
\theoremstyle{plain}\newtheorem{example}{Example}[section]
\theoremstyle{definition}\newtheorem{definition}{Definition}[section]
\theoremstyle{definition}\newtheorem*{notation}{Notation}
\theoremstyle{plain}\newtheorem{assumption}{Assumption}[section]
\theoremstyle{remark}

\makeatletter
\newcommand{\printfnsymbol}[1]{%
  \textsuperscript{\@fnsymbol{#1}}%
}
\makeatother

\begin{document}

%

%
\runningauthor{Marine Le Morvan, Nicolas Prost, Julie Josse, Erwan Scornet, Ga\"el Varoquaux}

\gdef\mytitle{Linear predictor on linearly-generated data with missing
values:
non consistency and solutions}
\twocolumn[

\aistatstitle{\mytitle}


\renewcommand*{\thefootnote}{\arabic{footnote}}
\addtocounter{footnote}{1}

\aistatsauthor{Marine Le Morvan \footnotemark[1]\footnotemark[2] \hspace{0.5em}%
Nicolas Prost\footnotemark[1] \hspace{0.5em}%
Julie Josse\footnotemark[1]\footnotemark[3] \hspace{0.5em}%
Erwan Scornet\footnotemark[3] \hspace{0.5em}%
Gaël Varoquaux\footnotemark[1]\footnotemark[4]}

\aistatsaddress{\footnotemark[1] Université Paris-Saclay, Inria, CEA, Palaiseau, 91120, France\\
\footnotemark[2] Université Paris-Saclay, CNRS/IN2P3, IJCLab, 91405 Orsay, France\\
\footnotemark[3] CMAP, UMR7641, Ecole Polytechnique, IP Paris, 91128 Palaiseau, France\\
\footnotemark[4] Mila, McGill University, Montréal, Canada
}
]

\setcounter{footnote}{0}

\begin{abstract}
We consider building predictors when the data have missing values. We study the seemingly-simple case where the target to predict is a linear function of the fully-observed data and we show that, in the presence of missing values, the optimal predictor is not linear in general. In the particular Gaussian case, it can be written as a linear function of multiway interactions between the observed data and the various missing-value indicators. Due to its intrinsic complexity, we study a
simple approximation and prove generalization bounds with finite samples, highlighting regimes for which each method performs best. We then show
that multilayer perceptrons with ReLU activation functions can be consistent, and can explore good trade-offs between the true model and approximations. Our study highlights the interesting family of models that are beneficial to fit with missing values depending on the amount of data available.
\end{abstract}

\section{Introduction}

Increasing data sizes and diversity naturally entail more and more missing
values. Data analysis with missing values has been extensively studied in
the statistical literature, with the leading work of
\cite{rubin1976inference}. However, this literature
\citep{little2002statistical, vanbuuren_2018, Rmistatic} does not address
the questions of modern statistical learning. First it
focuses on
estimating parameters and their variance, of a distribution --joint or
conditional-- as in the linear model
\citep{little1992regression, jones1996indicator}. This is typically done
using either
likelihood inference based on expectation maximization (EM) algorithms \citep{dempster1977maximum}
or multiple imputation \citep{vanbuuren_2018}. Second,  a large part of the literature only considers the restricted ``missing at random'' mechanism
\citep{rubin1976inference} as it allows maximum-likelihood
inference while ignoring the missing values distribution.
``Missing non at random'' mechanisms are much harder to address, and
the literature is thin, focusing on detailed models of the missingness
for a specific application such as collaborative filtering
\citep{marlin2009collaborative} or on
missing values occurring only on few variables \citep{kim2018data}. Statistical estimation often hinges on 
parametric models for the data and the missingness mechanism  \citep[except, for example,][]{mohan:pea19-r473}.
Finally, only few
notable exceptions study
supervised-learning settings, where the aim is to predict a target variable given input variables and the missing values are both in the training and the test sets \citep{zhang2005missing,pelckmans2005handling,liu2016,
josse2019consistency}. 
Machine-learning techniques have been extensively used to impute missing
values
\citep{lakshminarayan1996imputation,yoon2018gain}. However imputation is a different
problem from predicting a target variable and good imputation does not
always lead to good prediction \citep{zhang2005missing,josse2019consistency}.


As surprising as it sounds, even the linear model, the
simplest instance of regression models, has not been thoroughly studied
with missing values and reveals unexpected challenges.  This can be
explained because data with missing values can be seen mixed continuous (observed values) and categorical (missing-values indicators) data. In comparison to decision trees for instance, linear models are less
well-equipped by design to address such mixed data.

After establishing the problem of risk minimization in
supervised-learning settings with missing values, the first contribution of this paper is to develop the Bayes predictor under common Gaussian assumption. We highlight that the resulting problem of linear model with missing values is no longer linear.  
We use these results to introduce
two approaches to estimate a predictor, one based directly on the
Bayes-predictor expression, which boils down to performing one linear
model per pattern of missing values, and one derived from a linear
approximation,  which is  equivalent to imputing missing values by a constant and concatenating the pattern of missing values to the imputed design matrix. 
We derive new generalization bounds for these two estimates, therefore establishing the regimes in which each estimate has higher performance.
Due to the complexity of the learning task, we study the benefit of using multilayer perceptron (MLP), a good compromise between the complexity of the first approach and the extreme simplicity of the second one. We show its consistency with enough hidden units. 
Finally, we conduct experimental studies that show that the MLP often gives the best prediction and can appropriately  handle MNAR data. 

\section{Problem setting}

\begin{notation}[Missing values]
Throughout the paper, missing values are represented with the symbol $\tt
na$ satisfying, for all $x \in \mathds{R}^{\star}$, $\tt na \cdot x = \tt na$ and $\tt na \cdot 0 = 0$.
\end{notation}

Let us consider\footnote{Writing conventions used in this paper are detailed in appendix \ref{apx:notations}.}
a data matrix 
$\mathbf x_n \in \mathbb{R}^{n \times d}$ and a response vector $\mathbf y_n \in \mathbb R^n$, such as
\begin{align*}
    \mathbf x_n = \begin{pmatrix}
        9.1 & 8.5 \\
        2.1 & 3.5 \\
        6.7 & 9.6 \\
        4.2 & 5.5
    \end{pmatrix}&,~
    \mathbf y_n = \begin{pmatrix}
        4.6 \\
        7.9 \\
        8.3 \\
        4.6 
    \end{pmatrix}.
\end{align*}
However, 
only the incomplete design matrix $\mathbf z_n$ is available. 
%
%
%
Letting $\widetilde{\mathbb R} = \mathbb R\cup\{\tt na\}$, the
incomplete design matrix $\mathbf z_n$ belongs to $\widetilde{\mathbb
R}^{n \times d}$.
More precisely, denoting by $\mathbf m_n \in \{0,1\}^{n \times d}$ the positions of missing entries in $\mathbf z_n$ ($1$ if the entry is missing, $0$ otherwise),  $\mathbf z_n$ can be written as
    $\mathbf z_n = \mathbf x_n\odot(\mathbf 1 -\mathbf m_n) + {\tt na}\odot \mathbf m_n$, where $\odot$ is the term-by-term product.
In summary, the observed data are given by
\begin{flalign*}
    \mathbf z_n = \begin{pmatrix}
        9.1 & 8.5 \\
        2.1 & {\tt na} \\
        {\tt na} & 9.6 \\
        {\tt na} & {\tt na}
    \end{pmatrix},&&
    \mathbf m_n = \begin{pmatrix}
        0 & 0 \\
        0 & 1 \\
        1 & 0 \\
        1 & 1
    \end{pmatrix},&&
    \mathbf y_n = \begin{pmatrix}
        4.6 \\
        7.9 \\
        8.3 \\
        4.6
    \end{pmatrix},
\end{flalign*}
and are supposed to be $n$ i.i.d. realizations of generic random variables $Z,M,Y$.

\begin{notation}[Observed indices]
    For all values $m$ of a mask vector $M$, $obs(m)$ (resp. $mis(m)$) denote the indices of the zero entries of $m$ (resp. non-zero). For instance, if $z=(3.4, 4.1, {\tt na}, 2.6)$, then $m=(0, 0, 1, 0)$, $mis(m)=\{2\}$ and $obs(m)=\{0, 1, 3\}$.
\end{notation}

Throughout, the target $Y$ depends linearly on $X$, that is, there exist $\beta_0 \in \mathbb{R}$ and $\beta \in \mathbb{R}^d$ such that
\begin{align}
    Y &= \beta_0 + \langle X, \beta \rangle + \varepsilon,    \mbox{ where } \varepsilon\sim\mathcal N(0, \sigma^2).
\label{eq:predictor}
\end{align}

A natural loss function in the regression framework is the square loss
$\ell~:~ \mathbb R \times \mathbb R \rightarrow \mathbb R_+$ defined as
$\ell(y,y') = (y-y')^2$. The Bayes predictor $f^\star$ associated to this
loss is the best possible predictor, defined by
\begin{align}
    f^\star \in \underset{f:\widetilde{\mathbb R}^d \rightarrow \mathbb R}{\mathrm{argmin}}~
    R(f), \label{eq_optimisation_pb}
\end{align}
where
\begin{align*}
        R(f) \coloneqq \mathbb E[\ell(f(Z), Y)].
\end{align*}
Since we do not have access to the true distribution of $(X,Z,M,Y)$,
an estimate $\hat f$ is typically built by minimizing the  empirical risk
over a class of functions $\mathcal F$ \citep{vapnik1992principles}.
This is a well-studied problem in the complete case: efficient
gradient descent-based algorithms can be used to estimate predictors, and
there are many empirical and theoretical results on how to choose a good
parametric class of functions $\mathcal{F}$ to control the generalization
error. Because of the semi-discrete nature of $\widetilde{\mathbb
R}$, these results cannot be directly transposed to data with missing
values.

\section{Optimal imputation} \label{sec:optimalimputation}

The presence of missing values makes empirical risk minimization 
--optimizing an empirical version of  \eqref{eq_optimisation_pb}-- untractable. 
Indeed, $\widetilde{\mathbb R}^d$ is not a vector space, therefore incapacitating gradient-based algorithms. Hence, solving \eqref{eq_optimisation_pb} in presence of missing values requires a specific strategy.

\paragraph{Imputing by an optimal constant}
The simplest way to deal with missing values is to inject the incomplete data into $\mathbb R^d$. The easiest way to do so is to use constant imputation, \emph{i.e.} impute each feature $Z_j$ by a constant $c_j$: the most common choice is to impute by the mean or the median. However, it is also possible to optimise the constant with regards to the risk. 

\begin{proposition}[Optimal constant in linear model] \label{prop:opt_cst}
The imputation constants $(c_j^{\star})_{j \in \llbracket 1, d\rrbracket}$ optimal
to minimize the quadratic risk in a linear model can be easily computed by solving a linear
model with a design matrix constructed by
imputing $X$
with zeros and concatenating the mask $M$ as additional features (see
Appendix \ref{apx:notations}).
\end{proposition}

In an inferential framework, 
\citet{jones1996indicator} showed that constant imputation leads to
regression parameters that are biased compared to parameters on the
fully-observed data. 
We differ from Jones because our aim is prediction rather than
estimation. Indeed, minimizing a prediction risk with missing values  is different from recovering the behavior without missing values \citep{josse2019consistency}. 
Nevertheless, the strategy of replacing missing values with a constant
does not lead to Bayes-consistent predictors in the general setting, and
even under a Gaussian assumption as shown in Section
\ref{sec:bayespred}. In the general case, the problem can be illustrated by the following example which shows that the model is no longer linear when values are missing.


\paragraph{The best predictor need not be linear}

\begin{example}[Non-linear submodel]
\label{ex:nonlinear}
Let $Y = X_1 + X_2 + \varepsilon$, where $X_2 = \exp(X_1) +
\varepsilon_1$. Now, assume that only $X_1$ is observed. Then, the model can be rewritten as 
$$
Y = X_1 + \exp(X_1) + \varepsilon + \varepsilon_1,
$$
where $f(X_1) = X_1 + \exp(X_1)$ is the Bayes predictor. In this example, the submodel for which only $X_1$ is observed is not linear. 
\end{example}

From Example \ref{ex:nonlinear}, we deduce that there exists a large
variety of submodels for a same linear model. In fact, the submodel
structures depend on the structure of $X$ and on the missing-value
mechanism. Therefore, an extensive analysis seems unrealistic. Below, we
show that in the particular case of Gaussian generative mechanisms
submodels can be easily expressed, hence the Bayes predictor for each submodel can be computed exactly.  

\section{Bayes predictor}
\label{sec:bayespred}

We now derive the expression of $\mathbb{E}[Y|Z]$ under Model
(\ref{eq:predictor}) with missing values ($Z \in \widetilde{\mathbb R}^d$), as it gives the
Bayes-optimal predictor for the square loss \citep{bishop2006prml}.

The Bayes predictor can be written as  
\begin{align*}
    f^\star(Z) & = \mathds{E}\left[Y~|~Z\right] \\
    & = \mathds{E} \left[Y~|~M, X_{obs(M)}\right] \\
    & =  \sum_{m\in\{0,1\}^d}
\mathds{E}\left[Y | X_{obs(m)}, M=m\right] ~ \mathds 1_{M=m}.
\end{align*}
This formulation already highlights the
combinatorial issues: as suggested by \citet[Appendix
B]{rosenbaum_rubin_JASA1984}, estimating $f^\star(Z)$ may require to estimate $2^d$ different submodels.

As shown by Example \ref{ex:nonlinear}, controlling the form of $f^{\star}$ requires assumptions on 
the conditional relationships across the features $X_j$.
To ground our
theoretical derivations, we use the very common pattern mixture model
\citep{little1993pattern}, with Gaussian distributions:

\begin{assumption}[Gaussian pattern mixture model]\label{ass:lineardistr}
The distribution of $X$ conditional on $M$ is Gaussian, that is, for all $m \in \{0,1\}^d$, there exist $\mu_m$ and $\Sigma_m$ such that 
    \begin{align*}
    X~|~(M=m) &\sim \mathcal N(\mu_m, \Sigma_m).
    \end{align*}
\end{assumption}

A particular case of this distribution is the case where $X$ is Gaussian and independent of $M$.



\begin{proposition}[Expanded Bayes predictor]\label{prop:condexp}
Under Assumption \ref{ass:lineardistr} and Model (\ref{eq:predictor}), the
Bayes predictor $f^{\star}$ takes the form
\begin{align}
    f^\star(Z) &=  \langle W,  \delta \rangle  \label{prop_expand_Bayes1}, 
\end{align}
where the parameter $\delta \in \mathbb R^p$ is a function of $\beta$, $(\mu^m)_{m\in\{0, 1\}^d}$ and $(\Sigma^m)_{m\in\{0, 1\}^d}$, and the random variable $W \in \mathbb R^p$ is the concatenation of $j= 1, \hdots, 2^d$ blocks, each one being 
$$\left( \mathds 1_{M=m_j}~,~X_{obs(m_j)} \mathds 1_{M=m_j}\right).$$
\end{proposition}
An interesting aspect of this result is that the Bayes predictor is a
linear model, though not on the original data matrix $X$.
Indeed,  $W$ and $\delta$ are  vectors composed of $2^d$ blocks, for which only one block is ``switched on'' -- the one corresponding to the observed missing pattern  $M$. Elements of $W$ of this block are the observed values for $X$ and elements of $\delta$ of the same block are the linear coefficients corresponding to the observed missingness pattern. 
Equation (\ref{prop_expand_Bayes1}) can thus be seen as the concatenation of each of the $2^d$ submodels, where each submodel corresponds to a missing pattern. Denoting by $\mathbf x^{(m_j)}$ the design matrix $\mathbf{x_n}$ restricted to the samples with missing data pattern $m_j$ and to the columns $obs(m_j)$, then the matrix $W$ looks like: 
\begin{equation*}
    \begin{pmatrix}
        (1, \mathbf x^{(m_1)}) & 0 && 0 \\
        0 & (1, \mathbf x^{(m_2)}) & \cdots & 0 \\
        & \vdots && \\
        0 & 0&& (1, \mathbf x^{(m_{2^d})})
    \end{pmatrix}.
\end{equation*}

    The Bayes predictor can also be expressed in a second way, as shown in Proposition \ref{prop:factorized_bayes_predictor}.
    \begin{proposition}[Factorized Bayes predictor]
    \label{prop:factorized_bayes_predictor}
    We have
    \begin{align}\label{eq:factorized_bayes_predictor}
        f^\star(Z) = \sum_{\mathcal S \subset \llbracket 1, d\rrbracket}
        \left(
        \zeta_0^\mathcal S + \sum_{j=1}^d \zeta_j^\mathcal S(1-M_j) 
        X_j \right)
        \prod_{k\in\mathcal S}M_k,
    \end{align}
    where the parameter $\zeta \in \mathbb R^p$ is a function of $\delta$.
     In addition, one can write
    \begin{equation*}
        Y = f^\star(Z) + {\rm noise}(Z),
    \end{equation*}
    with ${\rm noise}(Z) = \varepsilon ~+ <\sqrt{T_M} \Xi, \beta_{mis(M)}>,$ and $\Xi\sim\mathcal N(0, I_d),$ 
    where $T_M = {\rm Var}\left(X_{mis(M)}|X_{obs(M)}, M\right)$ and $\sqrt{}$ denotes the square root of a positive definite symmetric matrix.
    \end{proposition}
    
    Expression \eqref{eq:factorized_bayes_predictor} is a polynome of $X$ and cross-products of $M$. 
As such, it is more convenient than expression \eqref{prop_expand_Bayes1} to compare to simpler estimates, as it can be truncated to low-order terms. This is done hereafter. Note that the multiplication $(1-M_j)X_j$ means that missing terms in $X_j$ are imputed by zeros.

Proofs of Proposition~\ref{prop:condexp} and
\ref{prop:factorized_bayes_predictor} can be found in the Appendix \ref{apx:proofs}. Thanks to these explicit expressions, the Bayes risk can be computed exactly as shown in Appendix \ref{apx:bayesrisk}. The value of the Bayes risk is extensively used in the Experiments (Section \ref{sec:emp_results}) to evaluate the performance of the different methods. 

The integer $p$ in equation \eqref{prop_expand_Bayes1} is the total number of parameters of the model which can be calculated by considering every sublinear model:
\begin{equation}
\label{eq:nbparameters}
    p = \sum_{k=0}^d \dbinom{d}{k} \times (k + 1) = 2^{d-1}\times(d+2).
\end{equation}
 
Strikingly, the Bayes predictor gathers $2^d$ submodels. When $d$ is not
small, estimating it from data is therefore a high-dimensional problem,
with computational and statistical challenges. For this reason, we
introduce hereafter a low-dimensional linear approximation of $f^{\star}$, without interaction terms. 
Indeed, the expression in Proposition \ref{prop:condexp} is not linear in
the original features and their missingness, but rather entails a
combinatorial number of
non-linearly derived features.

\begin{definition}[Linear approximation]
\label{def:linear_approx}
We define the linear approximation of $f^{\star}$ as
\begin{align*}
    f^{\star}_{\trunc}(Z) = \beta^{\star}_{0,0} + \sum_{j=1}^d \beta^{\star}_{j,0} \mathds{1}_{M_j=1} + \sum_{j=1}^d \beta^{\star}_{j} X_j \mathds{1}_{M_j=0}
\end{align*}
\end{definition}


$f^{\star}_{\trunc}(Z)$ is a linear function of the concatenated vector $(X, M)$ where $X$ is imputed by zeros, enabling a study of linear regression with that input. Note that this approximation is the same as defined in Proposition~\ref{prop:opt_cst}.

\section{Finite sample bounds for linear predictors}

The above expression of the Bayes predictor leads to two estimation
strategies with linear models. The first model is the direct empirical equivalent of the
Bayes predictor, using a linear regression to estimate the terms in the expanded Bayes predictor (Proposition \ref{prop:condexp}). It is a rich
model, powerful in low dimension, but it is costly and has large variance
in higher dimension. The second model is the approximation of the first
given in Definition \ref{def:linear_approx}. It is
has a lower approximation capacity but also a smaller variance since it contains fewer parameters.

For the theoretical analysis, we focus in this Section on the risk
between the estimate and the Bayes predictor $f^{\star}(Z)$. 
We thus consider the new framework below to handle our analysis. 

\begin{assumption}
\label{ass:th_finite_sample}
We have $Y = f^{\star}(Z) + \textrm{noise}(Z)$
as defined in Section~\ref{sec:bayespred}, where $\textrm{noise}(Z)$ is a centred noise conditional on $Z$ and such that there exists $\sigma^2>0$ satisfying $\mathds{V}[Y|Z] \leq \sigma^2$ almost surely. Besides, assume that $\|f^{\star}\|_{\infty}< L$ and  $\textrm{Supp}(X) \subset [-1,1]^{d}$. 
\end{assumption}
For all $L>0$ and for all function $f$, we define the clipped version $T_L f$ of $f$ at level $L$ by, for all $x$,
\begin{align*}
T_L f (x) =
\left\lbrace 
\begin{array}{lll}
     &  f(x) & \textrm{if}~|f(x)|\leq L \\
     &  L~\textrm{sign}(f(x)) & \textrm{otherwise}
\end{array}
\right.
\end{align*}

\subsection{Expanded linear model}

The expanded linear model is well specified, as the Bayes predictor detailed in Proposition \ref{prop:condexp} belongs to the model.

\begin{theorem}[Expanded model] \label{th:full_model_fin_sample}
Grant Assumption \ref{ass:th_finite_sample}. Let
$f_{\hat{\beta}_{\full}}$ be the linear regression estimate
 for the
expanded model (see Proposition \ref{prop:condexp}) computed via Ordinary Least Squares (OLS). Then, the risk of its
predictions clipped at $L$ satisfies 
\begin{align*}
    R(T_L f_{\hat{\beta}_{\full}}) & \leq c \max\{\sigma^2, L^2\} \frac{2^{d-1}(d+2) (1+\log n)}{n}  \\
    & \quad + \sigma^2.
\end{align*}

Additionally,  let $\mathcal{M} = \{m \in \{0,1\}^d, \mathds{P}[M=m]>0\}$ be the set of all possible missing patterns and assume that there exists $\alpha \in (0,1]$ such that $|\mathcal{M}| = \alpha 2^d$. Then, there exists a constant $c_1>0$, such that for all $n$ large enough, 
\begin{align*}
    R(T_L f_{\hat{\beta}_{\full}}) & \geq \sigma^2  + \frac{2^d c_1}{n+1}.
\end{align*}
\end{theorem}
The proof is provided in Appendix~\ref{sec:proof_bound}. Theorem~\ref{th:full_model_fin_sample} implies that the excess risk of the linear estimate of the expanded model is of order $\mathcal{O}(2^d/n)$
which grows exponentially fast with the original dimension of the problem.

\subsection{Constant imputation via linear approximation}

\begin{theorem}[Linear approximation model]\label{th:approx_model_fin_sample}
Grant Assumption \ref{ass:th_finite_sample}. Let $
f_{\hat{\beta}_{\trunc},L}$ be the linear regression estimate for the
approximated model (Definition \ref{def:linear_approx}),  computed via OLS. Then, the risk
of its predictions clipped at $L$ satisfies
\begin{align*}
    R(T_L f_{\hat{\beta}_{\trunc}}) & \leq \sigma^2 + c \max\{\sigma^2, L^2\} \frac{2d (1+\log n)}{n} \\
    &  \quad + 64 (d+1)^2 L^2
\end{align*}
\end{theorem}
The proof is provided in Appendix~\ref{sec:proof_bound}.
A direct consequence of Theorem \ref{th:approx_model_fin_sample} is that
the excess risk of the linear approximation of $f^{\star}$
(Definition\,\ref{def:linear_approx}) is 
$\mathcal{O}(d^2 + d/n)$.

Comparing this upper bound to the one obtained for the expanded model, we see that the risk of the expanded model is lower than that of the approximated model when $n \gg \frac{2^d}{d}$.
Therefore, the expanded model is only useful for very small data set (large $n$, small $d$), but for data sets of reasonable size, the linear approximation may be preferred. Nevertheless, as detailed in Section \ref{sec:Multilayer}, multilayers neural nets can be used as a compromise between both approaches. 

\section{Multilayer perceptron}\label{sec:Multilayer}



\subsection{Consistency}

\begin{theorem}[MLP]\label{th:MLP}
Assume that the Bayes predictor takes the form described in proposition~\ref{prop:condexp}, and that the support of $X$ is finite. A MLP i) with one hidden layer containing $2^d$ hidden units, ii) ReLU activation functions iii) which is fed with the concatenated vector $(X \odot (\mathbf 1-M), M)$, is Bayes consistent.
\end{theorem}

The complete proof of Theorem \ref{th:MLP} is given in Appendix 
\ref{sec:proof_mlp}, but we provide here the main ideas. The activation
for each hidden unit is a linear function of a design matrix constructed
by imputing $X$ with zeros and concatenating with the mask $M$. Thus,
just like for the optimal imputation problem of Proposition
\ref{prop:opt_cst}, the weights of the linear function can be seen as
either regular linear regression weights for the observed variables or
learned imputation constants for the missing ones. In the context of a
MLP with $2^d$ hidden units, we have $2^d$ such linear functions, meaning
that each hidden unit is associated with one learned imputation vector.
It turns out, as shown in the proof, that it is possible to choose the
imputation vector of each hidden unit so that one hidden unit is always
activated by points with a given missing-values pattern $m$ but never by
points with another missing-values pattern $m^\prime \ne m$. As a result,
all points with a given missing-values pattern fall into their own affine region, and it is then possible to adjust the weights so that the slope and bias of each affine region equals those of the Bayes predictor.\\

The number of parameters of a MLP with one hidden layer and $2^d$ units
is $(d+1)2^{d+1}+1$. Compared to the expanded Bayes predictor which has
$(d+1)2^{d-1}$ parameters, this is roughly 4 times more. This comes from
the fact that the MLP does not directly estimate the slope and bias of a
linear model per missing-values pattern. First, for each affine region
associated to a missing-values pattern, it estimates a slope for all
variables, and not only for the observed ones. This doubles the number of
parameters to be estimated. Second, it also needs to estimate the
imputation constants for each hidden unit (or equivalently missing-values
pattern) which again doubles the number of parameters to be estimated. As
a result, the MLP should require more samples than the expanded Bayes
predictor to achieve the Bayes rate. However, as we discuss below the parametrization of the MLP provides a natural way to control the capacity of the model.  By contrast, there is no such easy and natural way to control the capacity of the expanded Bayes predictor.

\subsection{Trading off estimation and approximation error.}
The prediction function of a MLP with one hidden layer and $n_h$ hidden
units is a piecewise affine function with at most $\sum_{j=0}^d {n_h \choose j}$ regions \citep{Pascanu2013OntheNumber}. Thus,
choosing $n_h = d$, we obtain $2^d$ affine regions, so potentially one
per missing-value pattern. However, the slopes and biases of these affine
regions are not independent, since they are linear combinations of the
weights associated to each hidden unit. Yet, if the data-generating process
has more regularities, $2^d$ different slopes may not be necessary to
approximate it well. Varying the number of hidden units $n_h$ thus
explores an interesting tradeoff between model complexity --which comes
at the cost of estimation error-- and approximation error, to successfully address medium sample sizes problems.

\section{Empirical study}
\label{sec:emp_results}

We now run an empirical study to illustrate our theoretical results, but
also to explore how the different bias-variance trade-offs of the various
models introduced affect prediction errors depending on the
amount of data available. The code to reproduce the experiments in available on GitHub \footnote{https://github.com/marineLM/linear\_predictor\_missing}.

\subsection{Experimental settings}

\paragraph{Simulation models}
The data $(X, M)$ is generated according to three different models. Two of them are instances of Assumption \ref{ass:lineardistr} while the third one is a classical Missing Non At Random (MNAR) model \citep{little2002statistical}.
\begin{description}
\item[mixture 1] The first model assumes that the data is generated
according to Assumption \ref{ass:lineardistr} with only one Gaussian
component shared by all missing-values patterns. This boils down to a classical Missing Completely At Random (MCAR) setting, where $X \sim \mathcal N(X| \mu,  \Sigma)$ and missing values are introduced uniformly at random, independent of $X$.
\item[mixture 3] The second model assumes that the data is generated
according to Assumption \ref{ass:lineardistr} with three Gaussian components. Each
missing-values pattern is associated to one of the three Gaussian
components in such a way that each component is associated with
the same number of missing-values patterns.
\item[selfmasking] The last model assumes that the data is generated
according to a single Gaussian, and that missing values are introduced
according to an inverse probit model parametrized by $\lambda$ and $\mu_0$ as
$P(M=1|X_j) = \Phi(\lambda_j(X_j - \mu_0))$. This model allows
to increase the probability of introducing missing values when the
variable increases (or decreases), hence the denomination
\emph{selfmasking}. It is a classical instance of a Missing Non At Random
(MNAR) problem. Estimation in MNAR settings is notoriously difficult as
most approaches, such as EM, rely on ignoring --marginalizing-- the
unobserved data which then introduces biases. 
\end{description}

For the three models, covariances for the Gaussian distributions are obtained as $BB^\top + D$ where $B \in \mathbb R^{d \times \lfloor\frac{d}{2}\rfloor}$ is drawn from a standard normal distribution and $D$ is a diagonal matrix with small positive elements to make the covariance matrix full rank. This gives covariance matrices with some strong correlations.

For \emph{mixture 1} and \emph{mixture 3}, missing values are introduced
in such a way that each missing-values pattern in equiprobable. For \emph{selfmasking}, the parameters of the masking function are chosen so that the missing rate for each variable is 25\%.

The response $Y$ is generated by a linear combination of the input variables as in Equation~\ref{eq:predictor}. Note that to generate $Y$, we use the complete design matrix (without missing values). In these experiments, the noise $\varepsilon$ is set to 0 and the regression coefficients $\beta, \beta_0$ are chosen equal to $\beta_0 = 1$ and $\beta = (1, 1, \dots, 1)$.


\paragraph{Estimation approaches}
For these three simulation scenarios, we compare four approaches:
\begin{description}
    \item[ConstantImputedLR:] optimal imputation method described in
    Proposition~\ref{prop:opt_cst} The regression is performed using ordinary least squares.
    
    \item[EMLR:] EM is used to fit a multivariate normal distribution for the $p+1$-dimensional random variable $(X_1, \dots, X_p, Y)$.
    Denoting by $\mu = (\mu_X, \mu_Y) \in \mathbb R^{p+1}$ the estimated mean and by $\Sigma$ the estimated covariance matrix (with blocks $\Sigma_{XX} \in \mathbb R^{p \times p}$, $\Sigma_{XY} \in \mathbb R^{p \times 1}$, and $\Sigma_{YY} \in \mathbb R$), the predictor used is:
    \begin{multline*}
        \mathbb E \left\{Y|X_{obs(M)}, M\right\} = \mu_Y + \\
        \Sigma_{Y, obs(M)} \Sigma_{obs(M)}^{-1} \left(X_{obs(M)} - \mu_{obs(M)}\right)
    \end{multline*}
    as can be obtained from the conditional Gaussian formula. Since EM directly estimates $\beta$, $\mu$ and $\Sigma$, it only has to estimate $(d+1)+d+\frac{d(d+1)}{2}=\frac{d(d+5)}{2}+1$ parameters, while ExpandedLR needs to estimate $2^{d-1} (d+2)$ parameters. However, this method is expected to perform well only in the \emph{mixture 1} setting.

    \item[MICELR:] Multivariate Imputation by Chained Equations (MICE) is a widely used and effective conditional imputation  method. We  used  the  \texttt{IterativeImputer} of scikit-learn which adapts MICE to out-of-sample imputation. 
    
    \item[ExpandedLR:] full linear model described in
Proposition~\ref{prop:condexp}. When the number of samples is small
compared to the number of missing-values patterns, it may happen that for a
given pattern, we have fewer samples than parameters to estimate. For this
reason, we used ridge regression on the expanded feature set, and the
regularization parameter was chosen by cross-validation over a small grid
of values ($10^{-3}$, $1$, $10^3$). The data is standardized before fitting the model.
    
    \item[MLP:] Multilayer perceptron with one hidden layer whose size is
varied between and 1 and $2^d$ hidden units. The input that is fed to the
MLP is $(X, M)$ where $X$ is imputed with zeros and $M$ is the mask.
Rectified Linear Units (ReLU) were used as activation functions. The MLP
was optimized with Adam, and a batch size of 200 samples. Weight decay
was applied and the regularization parameter chosen by cross-validation over a small grid of values ($10^{-1}, 10^{-2}, 10^{-4}$). The data is standardized before fitting the model.
\end{description}

\begin{figure*}[t!]
    \includegraphics[height=.33\linewidth]{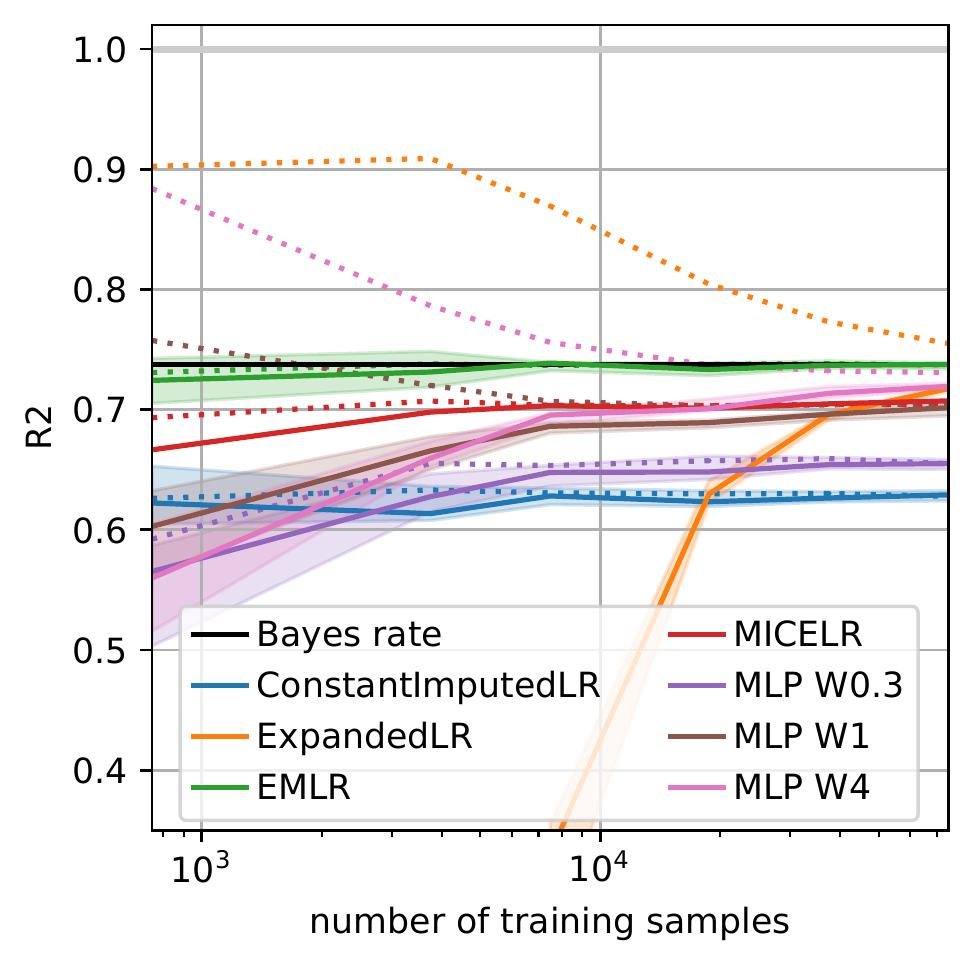}%
    \llap{\raisebox{.322\linewidth}{\sffamily\small{\bfseries Mixture 1}
	  (MCAR)\qquad\qquad}}%
    \llap{\raisebox{1ex}{\sffamily\bfseries a\hspace*{.3\linewidth}}}
    \hspace*{-1ex}%
    \hfill%
    \includegraphics[height=.33\linewidth]{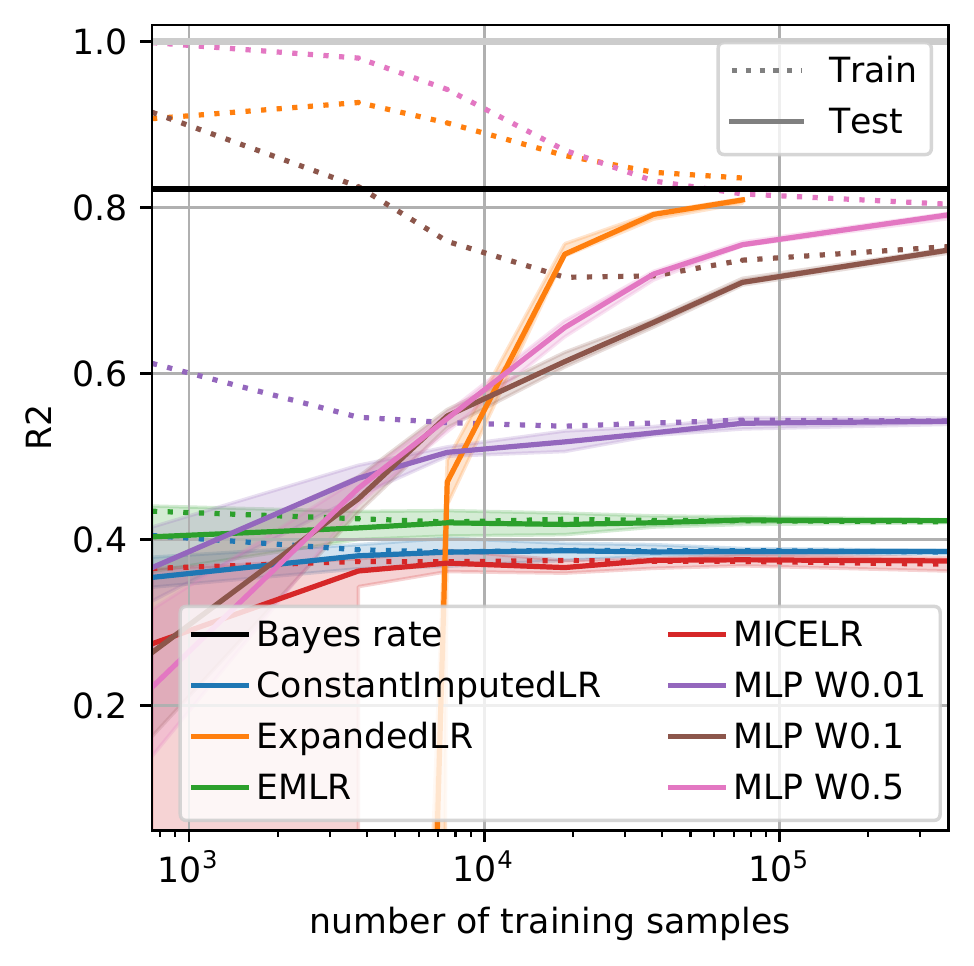}%
    \llap{\raisebox{.322\linewidth}{\sffamily\bfseries\small Mixture 3
	    \quad\qquad\qquad}}%
    \llap{\raisebox{1ex}{\sffamily\bfseries b\hspace*{.3\linewidth}}}
    \hspace*{-1ex}%
    \hfill%
    \includegraphics[height=.33\linewidth]{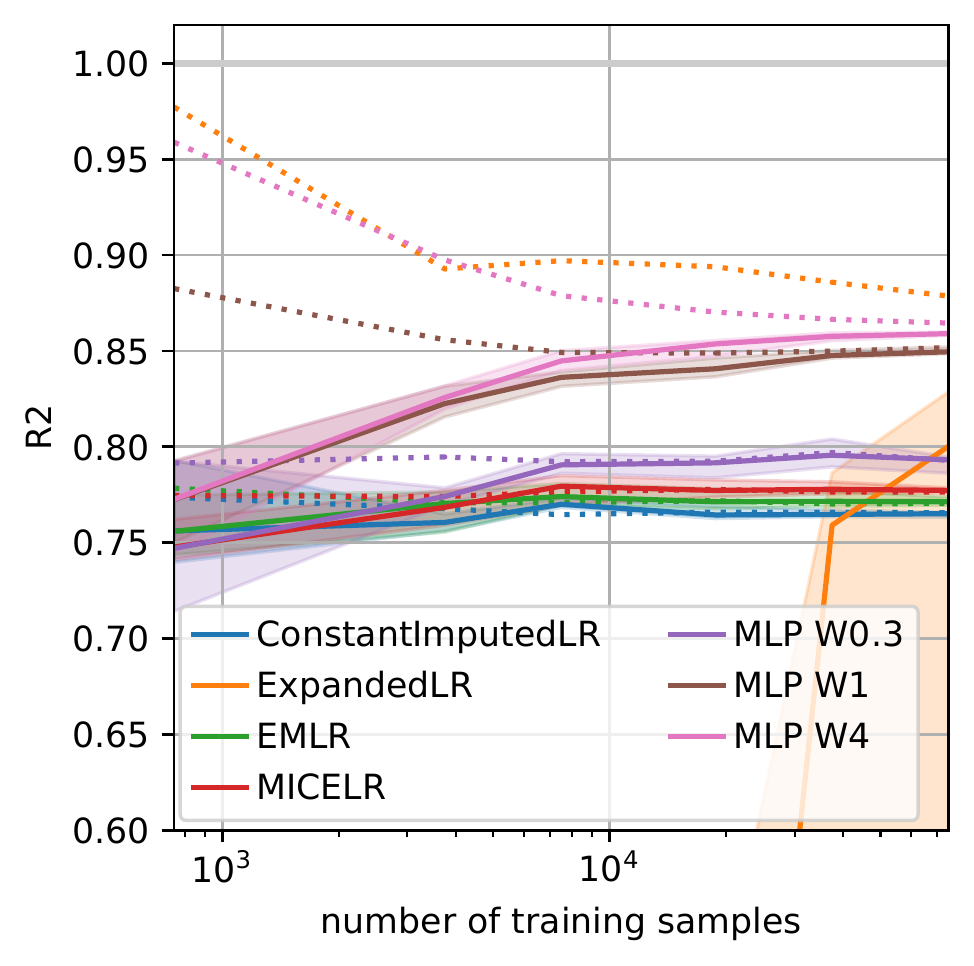}%
    \llap{\raisebox{.322\linewidth}{\sffamily\small{\bfseries Self-masked}
	(MNAR)\quad\qquad}}%
    \llap{\raisebox{1ex}{\sffamily\bfseries c\hspace*{.3\linewidth}}}
    \caption{\textbf{Learning curves} Train and test R2 scores
(respectively in dotted and in plain lines) as a
function of the number of training samples, for each data type.
Experiments were carried out in dimension $d=10$. The curves display the mean and 95\% confidence interval over 5 repetitions. The black horizontal line represents the Bayes rate (best achievable performance).
\autoref{fig:boxplots}, in the supp.\,mat.~gives a box plot 
at $n=75\,000$.}
    \label{fig:learning_curves}
\end{figure*}

\begin{figure*}[t!]
    \includegraphics[height=.29\linewidth]{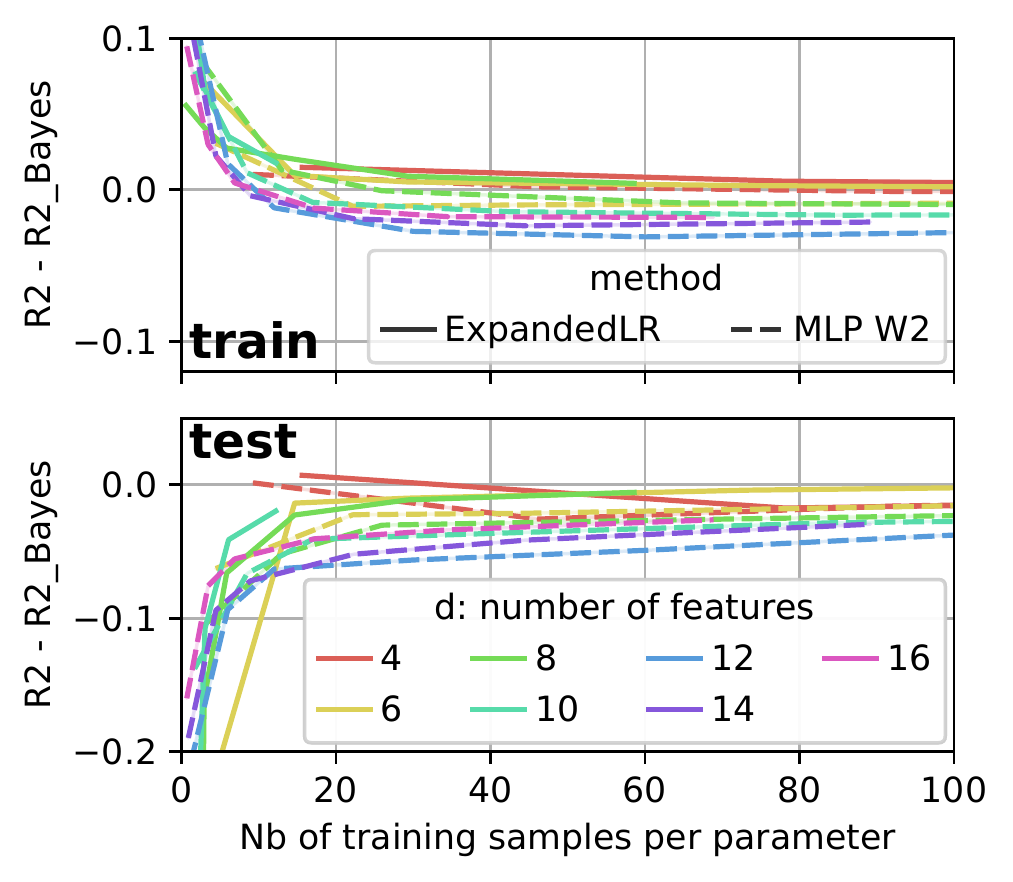}%
    \llap{\raisebox{.282\linewidth}{\sffamily\small{\bfseries Mixture 1}
	    (MCAR)\qquad\qquad}}%
    \llap{\raisebox{1ex}{\sffamily\bfseries a\hspace*{.3\linewidth}}}
    \hfill%
    \includegraphics[height=.29\linewidth]{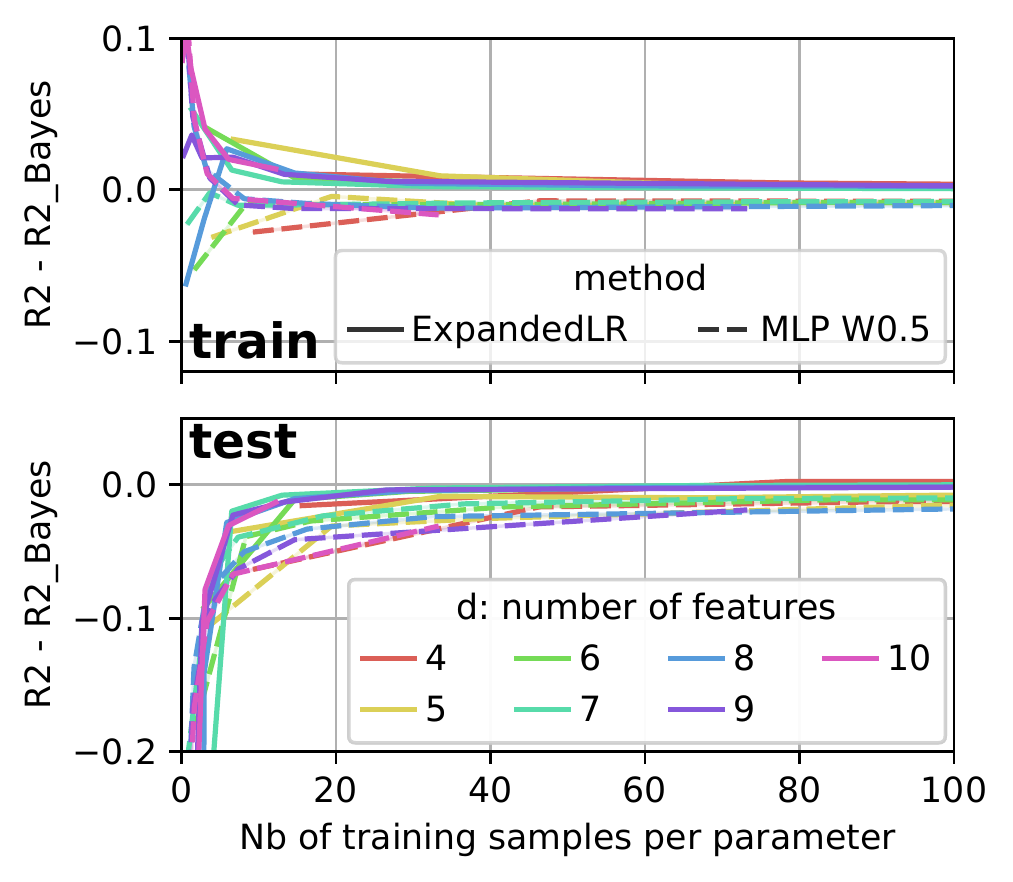}%
    \llap{\raisebox{.282\linewidth}{\sffamily\bfseries\small Mixture 3
	    \quad\qquad\qquad}}%
    \llap{\raisebox{1ex}{\sffamily\bfseries b\hspace*{.3\linewidth}}}
    \hfill%
    \includegraphics[height=.29\linewidth]{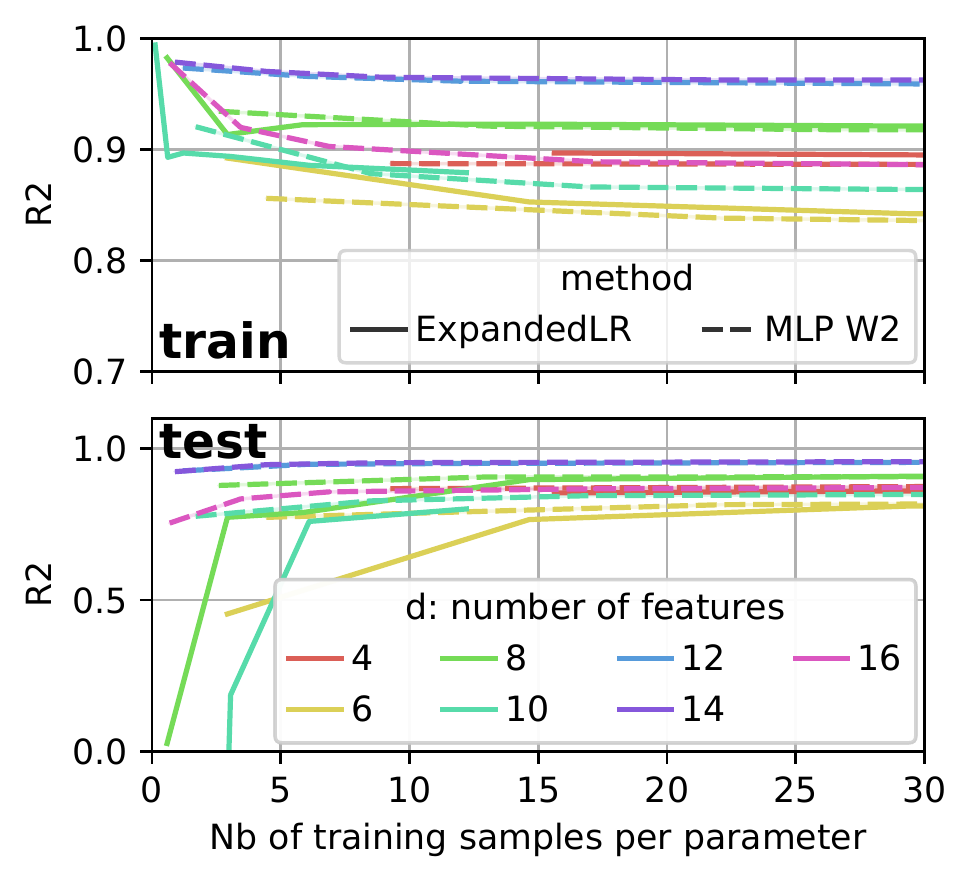}%
    \llap{\raisebox{.282\linewidth}{\sffamily\small{\bfseries
	Self-masked} (MNAR)\quad\qquad}}%
    \llap{\raisebox{1ex}{\sffamily\bfseries c\hspace*{.3\linewidth}}}
    \caption{\textbf{Learning curves varying the dimensionality} for the MLPs and ExpandedLR. For a given number of features $d$, the number of hidden units used is $2d$ for mixture 1 and self-masked and $0.5\times2^d$ for mixture 3.}
    \label{fig:scaling_n}
\end{figure*}

When referring to a MLP in the figures, we use the notation
\emph{MLP\_Wx} where x is a value indicating the number of units used.
For an experiment in dimension $d$, \emph{MLP\_Wx} refers to a MLP with
$x \times 2^d$ hidden units for \emph{mixture 3}, or a MLP with $x \times
d$ hidden units for \emph{MCAR} and \emph{MNAR}. The reason why we use
this notation is because to achieve the same performance level for
different dimensions, we must use a number of hidden units that is proportional to $2^d$
or to $d$ according to the data type (See Appendix
\ref{sec:scaling_q}). In the experiments, we test three different hidden layer sizes, to compare low, medium and high capacities. 


All models are compared to the Bayes rate computed in Appendix
\ref{apx:bayesrisk} whenever possible, i.e. for \emph{mixture 1} and
\emph{mixture 3}. In all experiments, datasets are split into train and test set (75\% train and 25\% test) and the performance reported is the R2 score of the predictor on the test set (or the difference between the R2 score of the predictor and that of the Bayes predictor).

\subsection{Results}

\paragraph{\emph{Mixture 1} and \emph{Mixture 3}}
We first focus on the data types satisfying
Assumption~\ref{ass:lineardistr}, i.e, \emph{mixture 1} and \emph{mixture 3}, since the theory developed in this paper applies for these cases.

Figure~\ref{fig:learning_curves} ({\sffamily a} and {\sffamily b})
presents the learning curves for the
various methods, as the number of samples increases from $10^3$ to $10^5$ ($5\times10^5$ for \emph{mixture 3})
and the dimension is fixed to $d=10$. First of all, this figure
experimentally confirms that ExpandedLR and the MLP are Bayes consistent.
With enough samples, both methods achieve the best possible performance.
It also confirms that ExpandedLR cannot be used in the small $n$ large
$d$ regime. Indeed, between $10,000$ and $20,000$ samples are required
for ExpandedLR to reach acceptable performances. 

Overall, the learning curves (Figure~\ref{fig:learning_curves}) highlight
three sample size regimes. We have a small sample size regime ($n <
1,000)$ where EM is the best option. Indeed, EM is Bayes consistent for
\emph{mixture 1}. For \emph{mixture
3}, EM performs badly but is
not worse than the other methods in the small sample size regime. It is
slightly better than ConstantImputedLR which still remains a reasonable
option in this regime. MICELR follows the behavior of EM with slightly
worse performance.

For $n>30,000$ in \emph{MCAR} and $n>10,000$ in \emph{mixture 3}, we are in a large sample size regime where ExpandedLR is an excellent choice, with performances on par or better than those of the MLP. The observation of small and large sample regimes support the theoretical analysis of Section~5. The fact that ExpandedLR outperforms the MLP for a larger sample range in \emph{mixture 3} ($n>10,000$) compared to \emph{mixture 1} ($n>30,000$) is explained by the fact that, to reach a given performance, the MLP needs fewer parameters in \emph{mixture 1} than in \emph{mixture 3}, and thus fewer samples.

Finally, for $1,000 < n < 10,000$ or $1,000 < n < 30,000$, we have a last regime where the MLP should be the preferred option, since it outperforms both ConstantImputedLR and ExpandedLR. It shows that for medium size samples, it is possible to adapt the width of the hidden layer to reach a beneficial compromise between estimation and approximation error. This is particularly useful since many real datasets fall into this medium size sample regime.

Figure~\ref{fig:scaling_n} demonstrates that the sample complexity is
directly related to the number of parameters to learn. In particular, it shows that ExpandedLR requires around 15 samples per parameter to achieve the Bayes rate (whatever the dimension). Since in dimension 10, the number of parameters of this model is $2^{d-1}(d+2) = 6144$, we need $15\times 2^{d-1}(d+2) \approx 100,000$ samples to achieve the Bayes rate and around $10,00$ samples to achieve a reasonable performance. By comparison, the MLP requires as many samples per parameter as ExpandedLR but it doesn't need to have as many parameters as ExpandedLR to perform well. For example in figure~\ref{fig:learning_curves} (\emph{mixture 1}), \emph{MLP W1} has $2d(d+1)+1 = 111$ parameters, which suffice to obtain good performances.

\paragraph{\emph{Self-masked (MNAR)}}
Self-masked (MNAR) does not satisfy Assumption~\ref{ass:lineardistr}.
Therefore under this data generating scheme, the expression of the Bayes
predictor derived in earlier sections is not valid, and ExpandedLR and
the MLP with $2^d$ hidden units need not be Bayes consistent.
Self-masking, where the probability of missingness depends on the
unobserved value, is however a classical missing-values mechanism, and it
is useful to assess the performance of the different methods in this
setting.  As shown in the right panel of Figure~\ref{fig:learning_curves},
the MLP outperforms all other methods for this data type. This reflects
the versatility of this method, which can adapt to all data generating schemes. ExpandedLR caps at a performance close to that of ConstantImputedLR, which highlights the dependence of this method on Assumption~\ref{ass:lineardistr}.
Imputation, explicit with MICE or implicit with EM, performs poorly. This
is expected as
in MNAR it does not ground good prediction \citep{josse2019consistency}.


\section{Discussion and conclusion}

We have studied how to minimize a prediction error on data where the
response variable is a linear function of a set of partially observed features. Surprisingly, with these missing values, the
Bayes-optimal predictor is no longer a linear function of the data. Under
Gaussian mixtures assumptions we derive a closed-form expression of this Bayes
predictor and used it to introduce a consistent estimation procedure of
the prediction function. However, it entails a very high-dimensional
estimation. Indeed, our generalization bounds --to our
knowledge the first finite-sample theoretical results on prediction with missing values-- show that the 
sample complexity scales, in general,
 as $2^d$. We therefore study several approximations,
in the form of constant imputation or a multi-layer perceptron (MLP),
which can also be consistent given a sufficient number of hidden units. A key
benefit of the MLP is that tuning its number of hidden units enables
reducing model complexity and thus decreasing the number of samples
required to estimate the model. Our experiments indeed show that 
in the finite-sample regime, using a MLP with a reduced number of hidden
units leads to the best prediction. Importantly, the MLP adapts naturally
to the complexity of the data-generating mechanism: it needs fewer hidden
units and less data to predict well in a missing completely at random
situation.

Our approach departs strongly from classical missing-values approaches,
which rely on EM or imputation to model unobserved values. Rather, we
tackle the problem with an empirical risk minimization strategy. An
important benefit of this approach is that it is robust to the
missing-values mechanism, unlike most strategies which require
missing-at-random assumptions. Our theoretical and empirical results
are useful to guide the choice of learning architectures in the presence
of missing-values: with a powerful neural architecture,
imputing with zeros and adding features indicating missing-values 
suffices.

\paragraph{Acknowledgments} Funding via ANR-17-CE23-0018
DirtyData and DataIA MissingBigData grants.

\bibliography{biblio}

\appendix

\onecolumn

  \hsize\textwidth
  \linewidth\hsize \toptitlebar {\centering
  {\Large{\bfseries  Supplementary materials} -- \mytitle \par}}
 \bottomtitlebar \vskip 0.2in plus 1fil minus 0.1in

\section{General remarks and proof of Proposition~\ref{prop:opt_cst}}
\label{apx:notations}

\subsection{Notation: letter cases}
   One letter refers to one quantity, with different cases: $U$ is a random variable, while $u$ is a constant. $\mathbf U_n$ is a (random) sample, and $\mathbf u_n$ is a realisation of that sample. $u_j$ is the $j$-th coordinate of $u$, and if $\mathcal J$ is a set, $u_{\mathcal J}$ denotes the subvector with indices in $\mathcal J$.

%
%
\subsection{Gaussian vectors}

In assumption \ref{ass:lineardistr}, conditionnally to $M$, $X$ is Gaussian. It is useful to remind that in that case, for two subsets of indices $\mathcal I$ and $\mathcal J$, conditional distributions can be written as
\begin{equation}\label{eq:conddistr}
    X_\mathcal I | (X_\mathcal J, M) \sim \mathcal N(\mu_{\mathcal I|\mathcal J}^M, \Sigma_{\mathcal I|\mathcal J}^M)
\end{equation}
with
\begin{flalign*}
    \left\{
        \begin{array}{rl}
            \mu_{\mathcal I|\mathcal J}^M &= \mu_\mathcal I^M+\Sigma_{\mathcal I\mathcal J}^M(\Sigma_{\mathcal J\mathcal J}^M)^{-1}(X_\mathcal J-\mu_\mathcal J^M)\\
            \Sigma_{\mathcal I|\mathcal J}^M &= \Sigma_{\mathcal I\mathcal I}^M-\Sigma_{\mathcal I\mathcal J}^M(\Sigma_{\mathcal J\mathcal J}^M)^{-1}(\Sigma_{\mathcal I\mathcal J}^M)^\mathsf{T}.
        \end{array}
    \right.
\end{flalign*}    
In particular, for all pattern $m$, for all $k\in mis(m)$,
\begin{flalign*}
    \mathbb E\left[X_k~\middle|~M=m, X_{obs(m)}\right]  = \mu_k^m + \Sigma_{k,obs(m)}^m\left(\Sigma_{obs(m)}^m\right)^{-1}\left(X_{obs(m)}-\mu_{obs(m)}^m\right).
\end{flalign*}

\subsection{Proof of Proposition~\ref{prop:opt_cst}}

Solving a linear regression problem with  optimal imputation constants $c^{\star} = (c_j^{\star})_{j \in \llbracket 1, d \rrbracket}$ can be written as 
\begin{align*}
 & (\beta^{\star}, c^{\star}) \in \textrm{argmin}_{ \beta, c \in \mathds{R}^d} 
 \mathds{E} \left[ \left(Y - \left(\beta_0 + \sum_{j=1}^d \beta_j \left(X_j \mathds{1}_{M_j=0} + c_j \mathds{1}_{M_j =1}\right) \right) \right)^2 \right] \\
 \Longleftrightarrow & (\beta^{\star}, c^{\star}) \in \textrm{argmin}_{ \beta, c  \in \mathds{R}^d} 
 \mathds{E} \left[ \left(Y - \left(\beta_0 + \sum_{j=1}^d \beta_j X_j \mathds{1}_{M_j=0} + \sum_{j=1}^d \beta_j c_j \mathds{1}_{M_j =1}\right)  \right)^2 \right],
 \end{align*}
where the terms $X_j \mathds{1}_{M_j=0}$ is equal to the variable $X_j$, imputed by zero if $X_j$ is missing and $\beta_j c_j$ is the linear coefficient associated to the variable $\mathds{1}_{M_j =1}$. Therefore, the linear regression coefficient $\beta^{\star} = (\beta_j^{\star})_{j \in \llbracket 1, d \rrbracket}$ and the optimal imputation constants $c^{\star} = (c_j^{\star})_{j \in \llbracket 1, d \rrbracket}$ can be solved via the linear regression problem with inputs $(X_j)_{j \in \llbracket 1, d \rrbracket},(\mathds{1}_{M_j=1})_{j \in \llbracket 1, d \rrbracket}$ where the first set of $d$ coefficients are the $(\beta_j^{\star})_{j \in \llbracket 1, d \rrbracket}$ and the second set of coefficients are equal to $(\beta_j^{\star} c_j^{\star})_{j \in \llbracket 1, d \rrbracket}$.

\section{Bayes estimate and Bayes risk}\label{apx:proofs}

\begin{proof}[Proof of Proposition \ref{prop:condexp}]
\begin{align*}
    \mathbb E[Y | Z]
    & = \mathbb E[\beta_0+\beta^\mathsf{T} X~|~Z] \\
    & = \mathbb E[\beta_0+\beta^\mathsf{T} X~|~M, X_{obs(M)}] \\
    & = \beta_0 + \beta_{obs(M)}^\mathsf{T} X_{obs(M)} + \beta_{mis(M)}^\mathsf{T}~\mathbb E[X_{mis(M)}~|~M, X_{obs(M)}] 
\end{align*}
where, by Equation~\ref{eq:conddistr},
\begin{align*}
    &\mathbb E[X_{mis(M)}~|~M, X_{obs(M)}] = \mu_{mis(M)}^M + \Sigma_{mis(M), obs(M)}^M\left(\Sigma_{obs(M)}^M\right)^{-1}\left(X_{obs(M)}-\mu_{obs(M)}^M\right).
\end{align*}
Hence,
\begin{align*}
    \mathbb E[Y | Z]
    & = \beta_0 + \beta_{mis(M)}^\mathsf{T}\left(\mu_{mis(M)}^M-\Sigma_{mis(M), obs(M)}^M\left(\Sigma_{obs(M)}^M\right)^{-1} \mu_{obs(M)}^M\right) & \\
    &\quad + \left(\beta_{obs(M)}^\mathsf{T}
    + \beta_{mis(M)}^\mathsf{T}~\Sigma_{mis(M), obs(M)}^M\left(\Sigma_{obs(M)}^M\right)^{-1}\right)X_{obs(M)} \\
    &= \delta_{obs(M), 0}^M + \left(\delta_{obs(M)}^M\right)^\mathsf{T} X_{obs(M)},
\end{align*}
by setting 
\begin{align*}
 \delta_{obs(M), 0}^M & = \beta_0 + \beta_{mis(M)}^\mathsf{T}\left(\mu_{mis(M)}^M-\Sigma_{mis(M), obs(M)}^M\left(\Sigma_{obs(M)}^M\right)^{-1} \mu_{obs(M)}^M\right) \\
 \delta_{obs(M)}^M & = \beta_{obs(M)} + \beta_{mis(M)}^\mathsf{T}~\Sigma_{mis(M), obs(M)}^M\left(\Sigma_{obs(M)}^M\right)^{-1}.
\end{align*}
Therefore, $E[Y|Z]$ takes the form, 
\begin{align*}
    \mathbb E[Y|Z]
    & = \sum_{m\in\{0, 1\}^d} \left[ \delta_{obs(m), 0}^m + \left(\delta_{obs(m)}^m\right)^\mathsf{T} X_{obs(m)} \right] \mathds 1_{M=m}\\
    &= \langle  W, \delta \rangle.
\end{align*}
\end{proof}

\begin{proof}[Proof of Proposition~\ref{prop:factorized_bayes_predictor}]
The polynomial expression is given by
\begin{align*}
    \mathbb E[Y|Z]
    &= \sum_{m\in\{0,1\}^d} 
    \mathds 1_{M=m} 
    \times\left(\delta_0^m + \sum_{j=1}^d \mathds 1_{j\in obs(m)} \delta_j^m X_j \right) \\
    &= \sum_{m\in\{0,1\}^d} 
    \prod_{k=1}^d  \left(1-(M_k-m_k)^2\right)
    \times\left(\delta_0^m + \sum_{j=1}^d (1-M_j) \delta_j^m X_j \right) \\
    &= \sum_{m\in\{0,1\}^d} 
    \prod_{k=1}^d  \left(1-M_k-m_k+2M_km_k\right)
    \times\left(\delta_0^m + \sum_{j=1}^d (1-M_j) \delta_j^m X_j \right) \\
    &= \sum_{m\in\{0,1\}^d} 
    \sum_{
        \scriptsize{\begin{array}{c}
            \mathcal S_1\sqcup\mathcal S_2\sqcup\mathcal S_3\\
            \sqcup\mathcal S_4=\llbracket 1,d\rrbracket
        \end{array}}}
    (-1)^{|\mathcal S_2|+|\mathcal S_3|}2^{|\mathcal S_4|}
    \prod_{
        \scriptsize{\begin{array}{c}
            k_3\in\mathcal S_3,\\
            k_4\in\mathcal S_4
        \end{array}}}
    m_{k_3}m_{k_4}
    \prod_{
        \scriptsize{\begin{array}{c}
            k_2\in\mathcal S_2,\\
            k_4\in\mathcal S_4
        \end{array}}}
    M_{k_2}M_{k_4}
    \times\left(\delta_0^m + \sum_{j=1}^d (1-M_j) \delta_j^m X_j \right) \\
    &\text{(where $\mathcal S_1\sqcup\mathcal S_2\sqcup\mathcal S_3\sqcup\mathcal S_4$ is a partition of $\llbracket 1, d\rrbracket$)}\\
    &=
    \sum_{
        \scriptsize{\begin{array}{c}
            \mathcal S_1\sqcup\mathcal S_2\sqcup\mathcal S_3\\
            \sqcup\mathcal S_4=\llbracket 1,d\rrbracket
        \end{array}}}
    (-1)^{|\mathcal S_2|+|\mathcal S_3|}2^{|\mathcal S_4|}
    \sum_{
        \scriptsize{\begin{array}{c}
            m\in\{0, 1\}^d \\
            obs(m) \subset\mathcal S_3^c \cap\mathcal S_4^c
        \end{array}}}
    1\times\prod_{k_2\in\mathcal S_2,k_4\in\mathcal S_4} M_{k_2}M_{k_4}
    \times\left(\delta_0^m + \sum_{j=1}^d (1-M_j) \delta_j^m X_j \right) \\
    &=
    \sum_{
        \scriptsize{\begin{array}{c}
            \mathcal S_1\sqcup\mathcal S_2\sqcup\mathcal S_3\\
            \sqcup\mathcal S_4=\llbracket 1,d\rrbracket
        \end{array}}}
    \prod_{k_2\in\mathcal S_2,k_4\in\mathcal S_4} M_{k_2}M_{k_4}
    \times\left(
    (-1)^{|\mathcal S_2|+|\mathcal S_3|}2^{|\mathcal S_4|}
    \sum_{
        \scriptsize{\begin{array}{c}
            m\in\{0, 1\}^d \\
            obs(m) \subset\mathcal S_3^c \cap\mathcal S_4^c
        \end{array}}}
    \left(\delta_0^m + \sum_{j=1}^d (1-M_j) \delta_j^m X_j \right)\right) \\
    &=
    \sum_{
        \scriptsize{\begin{array}{c}
            \mathcal S_1\sqcup\mathcal S_2\sqcup\mathcal S_3\\
            \sqcup\mathcal S_4=\llbracket 1,d\rrbracket
        \end{array}}}
    \prod_{k_2\in\mathcal S_2,k_4\in\mathcal S_4} M_{k_2}M_{k_4}
    \times\left(
    \zeta_0^{\mathcal S_2,\mathcal S_3,\mathcal S_4} + \sum_{j=1}^d (1-M_j) \zeta_j^{\mathcal S_2,\mathcal S_3,\mathcal S_4} X_j \right) \\
    &=
    \sum_{
        \scriptsize{\begin{array}{c}
            \mathcal S_2\sqcup\mathcal S_4\subset\llbracket 1,d\rrbracket
        \end{array}}}
    \prod_{k_2\in\mathcal S_2,k_4\in\mathcal S_4} M_{k_2}M_{k_4}
    \times
    \sum_{
        \scriptsize{\begin{array}{c}
            \mathcal S_1\sqcup\mathcal S_3=(\mathcal S_2\sqcup\mathcal S_4)^c
        \end{array}}}
    \left(
    \zeta_0^{\mathcal S_2,\mathcal S_3,\mathcal S_4} + \sum_{j=1}^d (1-M_j) \zeta_j^{\mathcal S_2,\mathcal S_3,\mathcal S_4} X_j \right) \\
    &\text{(reindexing $\mathcal S=\mathcal S_2\sqcup\mathcal S_4$)}\\
    &=
    \sum_{
        \scriptsize{\begin{array}{c}
            \mathcal S\subset\llbracket 1,d\rrbracket
        \end{array}}}
    \prod_{k\in\mathcal S} M_{k}
    \times
    \left(
    \zeta_0^{\mathcal S} + \sum_{j=1}^d (1-M_j) \zeta_j^{\mathcal S} X_j \right).
\end{align*}

Finally, the expression of ${\rm noise}(Z)$ results from
\begin{equation*}
    X_{mis(M)} | X_{obs(M)}, M=m \sim \mathcal N(\mu_M, T_M)
\end{equation*}
where the conditional expectation $\mu_M$ has been given above and
\begin{equation*}
    T_M = \Sigma_{mis(M)}-\Sigma_{mis(M),obs(M)}\left(\Sigma_{obs(M)}\right)^{-1}\Sigma_{obs(M),mis(M)}.
\end{equation*}

\end{proof}

\section{Bayes Risk}\label{apx:bayesrisk}

\begin{proposition}\label{prop:bayesrisk}
    The Bayes risk associated to the Bayes estimator of proposition \ref{prop:condexp} is given by 
\begin{equation*}
    \mathbb E\left[\left(Y - f^\star(Z) \right)^2\right]
    = \sigma^2 + \sum_{m\in\{0, 1\}^d} \mathbb P(M=m) \Lambda_m,
\end{equation*}
with 
\begin{align*}
\Lambda_m & = \left(\gamma_{obs(m)}^m\right)^\mathsf{T}\Sigma_{obs(m)}^m\gamma_{obs(m)}^m + \beta_{mis(m)}^\mathsf{T}\Sigma_{mis(m)}^m\beta_{mis(m)} 
- 2\left(\gamma_{obs(m)}^m\right)^\mathsf{T}\Sigma_{obs(m),mis(m)}^m\beta_{mis(m)} \\
& \quad + \left(\gamma_{obs(m),0}^m\right)^2 + \left(\left(\gamma_{obs(m)}^m\right)^\mathsf{T} \mu_{obs(m)}^m\right)^2  + \left(\beta_{mis(m)}^\mathsf{T} \mu_{mis(m)}^m\right)^2  + 2\gamma_{obs(m),0}^m\left(\gamma_{obs(m)}^m\right)^\mathsf{T}\mu_{obs(m)}^m \\
& \quad - 2\gamma_{obs(m),0}^m\beta_{mis(m)}^\mathsf{T}\mu_{mis(m)}^m 
- 2\left(\gamma_{obs(m)}^m\right)^\mathsf{T} \mu_{obs(m)}^m \beta_{mis(m)}^\mathsf{T}\mu_{mis(m)}^m,
\end{align*}
where $\gamma_{obs(m)}^m$ is a function of the regression coefficients on the missing variables and the means and covariances given $M$.
\end{proposition}

\begin{proof}[Proof of Proposition \ref{prop:bayesrisk}]
\begin{align*}
    &\mathbb E\left[\left(\mathbb E[Y|Z]-Y\right)^2\right] &\\
    &= \mathbb E\left[\left(\delta_{obs(M), 0}^M + \left(\delta_{obs(M)}^M\right)^\mathsf{T} X_{obs(M)} - \beta_0 - \beta^\mathsf{T} X - \varepsilon\right)^2\right] &\\
    &= \mathbb E\left[\left(\delta_{obs(M), 0}^M - \beta_0 + \left(\delta_{obs(M)}^M - \beta_{obs(M)}\right)^\mathsf{T} X_{obs(M)} - \beta_{mis(M)}^\mathsf{T} X_{mis(M)} - \varepsilon\right)^2\right].
\end{align*}
By posing
\begin{align*}
    \left\{
        \begin{array}{rll}
            \gamma_{obs(M),0}^M &= \delta_{obs(M), 0}^M - \beta_0 &= \beta_{mis(M)}^\mathsf{T}\left(\mu_{mis(M)}^M-\Sigma_{mis(M), obs(M)}^M\left(\Sigma_{obs(M)}^M\right)^{-1} \mu_{obs(M)}^M\right) \\
            \gamma_{obs(M)}^M &= \delta_{obs(M)}^M - \beta_{obs(M)} &= \beta_{mis(M)}^\mathsf{T}~\Sigma_{mis(M), obs(M)}^M\left(\Sigma_{obs(M)}^M\right)^{-1},
        \end{array}
    \right. 
\end{align*}
one has
\begin{align*}
    &\mathbb E\left[\left(\mathbb E[Y|Z]-Y\right)^2\right] &\\
    &= \mathbb E\left[\left(\gamma_{obs(M), 0}^M + \left(\gamma_{obs(M)}^M\right)^\mathsf{T} X_{obs(M)} - \beta_{mis(M)}^\mathsf{T} X_{mis(M)} - \varepsilon\right)^2\right] \\
    &= \sum_{m\in:\{0, 1\}^d} \mathbb P(M=m) \cdot \mathbb E\left[\left(\gamma_{obs(m), 0}^m + \left(\gamma_{obs(m)}^m\right)^\mathsf{T} X_{obs(m)} - \beta_{mis(m)}^\mathsf{T} X_{mis(m)} - \varepsilon\right)^2 ~\middle|~ M=m\right] \\
    &= \sum_{m\in:\{0, 1\}^d} \mathbb P(M=m) \cdot \Bigg[
    \sigma^2 + \mathrm{Var}\left(\left(\gamma_{obs(m)}^m\right)^\mathsf{T} X_{obs(m)} - \beta_{mis(m)}^\mathsf{T} X_{mis(m)} ~\middle|~ M=m\right) \\
    & \qquad + \left(\gamma_{obs(m),0}^m + \left(\gamma_{obs(m)}^m\right)^\mathsf{T} \mathbb E\left[X_{obs(m)}\middle|M=m\right] - \beta_{mis(m)}^\mathsf{T} \mathbb E\left[X_{mis(m)}\middle|M=m\right]\right)^2
    \Bigg] \\
    &= \sigma^2 + \sum_{m\in:\{0, 1\}^d} \mathbb P(M=m) ~ \Lambda_m
\end{align*}
\end{proof}

\section{Proof of Theorem~\ref{th:full_model_fin_sample} and Theorem~\ref{th:approx_model_fin_sample}}
\label{sec:proof_bound}

Theorem 11.3 in \cite{gyorfi2006distribution} allows us to bound the risk of the linear estimator, even in the misspecified case. We recall it here for the sake of completeness. 

\begin{theorem}[Theorem 11.3 in \cite{gyorfi2006distribution}]
Assume that 
\begin{align*}
    Y = f^{\star}(X) + \varepsilon,
\end{align*}
where $\|f^{\star}\|_{\infty} < L$ and $\mathds{V}[\varepsilon|X] < \sigma^2$
almost surely. Let $\mathcal{F}$ be the space of linear function $f:
[-1,1]^d \to \mathds{R}$. Then, letting $T_L {f}_n$ be 
the linear regression $f_n$ estimated via OLS, clipped at $\pm L$, we have
\begin{align*}
    \mathds{E}[(T_L {f}_n(X) - f^{\star}(X))^2] \leq c \max\{\sigma^2, L^2\} \frac{d (1+\log n)}{n} + 8 \inf_{f \in \mathcal{F}} \mathds{E}[ (f(X) - f^{\star}(X))^2],
\end{align*}
for some universal constant $c$.
\end{theorem}

\begin{proof}[Proof of Theorem \ref{th:full_model_fin_sample}]
Since Assumptions of Theorem 11.3 in \cite{gyorfi2006distribution} are satisfied, we have
\begin{align*}
    \mathds{E}[(T_L f_{\hat{\beta}_{\full}}(Z) - f^{\star}(Z))^2] \leq c \max\{\sigma^2, L^2\} \frac{p (1+\log n)}{n} + 8 \inf_{f \in \mathcal{F}} \mathds{E}[ (f(Z) - f^{\star}(Z))^2],
\end{align*}
Since the model is well-specified, the second term is null. Besides, 
\begin{align*}
\mathds{E}[(Y - T_L f_{\hat{\beta}_{\full}}(Z))^2]
& \leq \mathds{E}[(Y - f^{\star}(Z))^2] + \mathds{E}[(T_L f_{\hat{\beta}_{\full}}(Z) - f^{\star}(Z) )^2]\\
& \leq \sigma^2 + c \max\{\sigma^2, L^2\} \frac{p (1+\log n)}{n},
\end{align*}
which concludes the proof since the full linear model has $p = 2^{d-1}(d+2)$ parameters.

To address the second statement of Theorem~\ref{th:full_model_fin_sample}, recall that in our setting, the dimension $d$ is fixed and does not grow to infinity with $n$. Let $\mathcal{M} = \{m \in \{0,1\}^d, \mathds{P}[M=m]>0\}$ and, for all $m \in \mathcal{M},$ $N_m = |\{i: M_i = m\}|$. Note that, the estimator in Theorem~\ref{th:full_model_fin_sample} is nothing but $|\mathcal{M}|$ linear estimators, each one being fitted on data corresponding to a specific missing pattern $m \in \mathcal{M}$.  Thus, according to \cite{tsybakov2003optimal}, we know that, there exists constants $c_1, c_2>0$, such that, for each missing pattern $m \in \mathcal{M}$, we have the lower bound,
\begin{align*}
 \mathds{E}[(Y - T_L f_{\hat{\beta}_{\full}}(Z))^2 | M = m, N_m]  \geq  \sigma^2 + c_1 \frac{d+1 - \|m\|_0}{N_m}\mathds{1}_{N_m\geq 1} + c_2 \mathds{1}_{N_m=0}.
\end{align*}
Taking the expectation with respect to $N_m \sim B(n, \mathds{P}[M=m])$ and according to Lemma 4.1 in \cite{gyorfi2006distribution}, we have, for all $m \in \mathcal{M}$,
\begin{align*}
 \mathds{E}[(Y - T_L f_{\hat{\beta}_{\full}}(Z))^2 | M = m]  \geq \sigma^2 + c_1 \frac{2(d+1 - \|m\|_0)}{(n+1)\mathds{P}[M=m] } + c_2 (1-\mathds{P}[M=m])^n.
\end{align*}
Consequently, 
\begin{align*}
R(T_L f_{\hat{\beta}_{\full}}) & = \sum_{m \in \mathcal{M}} \mathds{E}[(Y - T_L f_{\hat{\beta}_{\full}}(Z))^2 | M = m] \mathds{P}[M = m]\\
& \geq \sigma^2 + \frac{2 c_1 }{n+1} \sum_{m \in \mathcal{M}} (d+1 - \|m\|_0) + c_2 \sum_{m \in \mathcal{M}} (1-\mathds{P}[M=m])^n \mathds{P}[M=m]\\
& \geq \sigma^2 + \frac{2 c_1 |\mathcal{M}|}{n+1} + c_2 (1- \min\limits_{m \in \mathcal{M}} \mathds{P}[M=m])^n.
\end{align*}
By assumption, there exists a constant $c$, such that, for all $n$ large enough, we have
\begin{align*}
R(T_L f_{\hat{\beta}_{\full}}) 
& \geq \sigma^2 + \frac{2^d c}{n+1}.
\end{align*}
\end{proof}

\begin{proof}[Proof of Theorem~\ref{th:approx_model_fin_sample}]
As above, 
\begin{align*}
    R(T_L f_{\hat{\beta}_{\trunc}}) & \leq \sigma^2 + \mathds{E}[(f^{\star}(Z) - T_L f_{\hat{\beta}_{\trunc}}(Z))^2]\\
&  \leq \sigma^2 + c \max\{\sigma^2, L^2\} \frac{2d (1+\log n)}{n} 
+ 8 \mathds{E}[(f^{\star}(Z) - f_{{\beta}^{\star}_{\trunc}}(Z))^2].
\end{align*}
To upper bound the last term, note that, for any $\beta_{\trunc}$ we have
\begin{align*}
& \mathds{E}[(f^{\star}(Z) - f_{{\beta}_{\trunc}}(Z))^2]\\
& = \mathds{E} \Bigg[ \beta_{0,0,\trunc}  + \sum_{j=1}^d \beta_{0,j,\trunc} \mathds{1}_{M_j=1} - \sum_{m \in \{0,1\}^d} \beta_{0,m, \full}^{\star} \mathds{1}_{M=m} \\
& \qquad + \Big( \beta_{1, \trunc} - \sum_{m \in \{0,1\}^d} \beta_{1,m, \full}^{\star} \mathds{1}_{M=m}  \Big) X_1 \\
& \qquad + \hdots  + \Big( \beta_{d, \trunc} - \sum_{m \in \{0,1\}^d} \beta_{d,m, \full}^{\star} \mathds{1}_{M=m}  \Big) X_d \Bigg]^2.
\end{align*}
Denoting by $X_{\trunc}$ the design matrix $X$ where each element $\texttt{na}$ has been replaced by zero, and using a triangle inequality, we have
\begin{align*}
& \mathds{E}[ (W \beta^{\star}_{\textrm{full}} - X_{\trunc} \beta_{\trunc})^2]\\
& \leq (d+1) \mathds{E}\Bigg[ \beta_{0,0,\trunc} + \sum_{j=1}^d \beta_{0,j,\trunc} \mathds{1}_{M_j=1} - \sum_{m \in \{0,1\}^d} \beta_{0,m, \full}^{\star} \mathds{1}_{M=m} \Bigg]^2\\
& \quad + (d+1) \sum_{j=1}^d \mathds{E}\Bigg[\Big( \beta_{j, \trunc} - \sum_{m \in \{0,1\}^d} \beta_{j,m, \full}^{\star} \mathds{1}_{M=m}  \Big) X_j \Bigg]^2 
\end{align*}
Now, set for all $j$, $\beta_{0,j,\trunc}=0$ and for all $j=1, \hdots, d$, 
\begin{align*}
    \beta_{j, \trunc} = \mathds{E} \left[ \sum_{m \in \{0,1\}^d} \beta_{j,m, \full}^{\star} \mathds{1}_{M=m} \right]
\end{align*}
and also
\begin{align*}
    \beta_{0,0,\trunc} = \mathds{E} \left[ \sum_{m \in \{0,1\}^d} \beta_{0,m, \full}^{\star} \mathds{1}_{M=m} \right].
\end{align*}
Therefore, for this choice of $\beta_{\trunc}$, 
\begin{align*}
& \mathds{E} [ (W \beta^{\star}_{\textrm{full}} - X_{\trunc} \beta_{\trunc})^2]\\
& \leq (d+1)  \mathds{V}\Bigg[\sum_{m \in \{0,1\}^d} \beta_{0,m, \full}^{\star} \mathds{1}_{M=m} \Bigg]  + (d+1) \|X\|_{\infty}^2 \sum_{j=1}^d \mathds{V}\Big[ \sum_{m \in \{0,1\}^d} \beta_{j,m, \full}^{\star} \mathds{1}_{M=m}  \Big] \\
& \leq 8(d+1)^2\|f^{\star}\|_{\infty}^2.
\end{align*}
Finally, by definition of $\beta^{\star}_{\trunc}$, we have
\begin{align*}
  \mathds{E}[(f^{\star}(Z) - f_{{\beta}^{\star}_{\trunc}}(Z))^2] & \leq \mathds{E}[(f^{\star}(Z) - f_{{\beta}_{\trunc}}(Z))^2] \\
& \leq 8(d+1)^2\|f^{\star}\|_{\infty}^2.
\end{align*}
Finally, 
\begin{align*}
    R(T_L f_{\hat{\beta}_{\trunc}}) 
& \leq  \sigma^2 + c \max\{\sigma^2, L^2\} \frac{d (1+\log n)}{n} 
+ 64 (d+1)^2 L^2,
\end{align*}
since $\|f^{\star}\|_{\infty} \leq L$, according to Assumption \ref{ass:th_finite_sample}.
\end{proof}

\section{Proof of Theorem~\ref{th:MLP}}
\label{sec:proof_mlp}%
Let $W^{(1)} \in \mathbb R^{2^d \times 2d}$ be the weight matrix connecting the input layer to the hidden layer, and $W^{(2)} \in \mathbb R^{2^d}$ the matrix connecting the hidden layer to the output unit. Let $b^{(1)} \in \mathbb R^{2^d}$ be the bias for the hidden layer and $b^{(2)} \in \mathbb R$ the bias for the output unit. With these notations, the activations of the hidden layer read:
\begin{align*}
    \forall k \in \llbracket 1, 2^d \rrbracket,\, a_k = W_{k,.}^{(1)}(X, M) + b^{(1)}_k
\end{align*}

Splitting $W^{(1)}$ into two parts $W^{(X)}, W^{(M)} \in \mathbb R^{2^d \times d}$, the activations can be rewritten as:
\begin{align*}
    \forall k \in \llbracket 1, 2^d \rrbracket,\, a_k = W_{k,.}^{(X)} X + W_{k,.}^{(M)} M + b^{(1)}_k
\end{align*}

\textbf{Case 1}: Suppose that $\forall k \in \llbracket 1, 2^d \rrbracket,\, \forall j \in \llbracket 1, d \rrbracket,\, W_{k,j}^{(X)} \ne 0$.\\
With this assumption, the activations can be reparametrized by posing $G_{k, j} = W^{(M)}_{k, j}/W^{(X)}_{k, j}$, which gives:
\begin{align*}
    \forall k \in \llbracket 1, 2^d \rrbracket,\, a_k &= W_{k,.}^{(X)} X + W_{k,.}^{(X)} \odot G_{k, .} M + b^{(1)}_k\\
    &= W_{k,obs(M)}^{(X)} X_{obs(M)} + W_{k,mis(M)}^{(X)} G_{k, mis(M)} + b^{(1)}_k
\end{align*}
and the predictor for an input $(x, m) \in \mathbb R^d \times \left\{0, 1 \right\}^d$ is given by:
\begin{align*}
    y(x, m) &= \sum_{k=1}^{2^d} W_k^{(2)} ReLU(a_k^{(1)}) + b^{(2)}\\
    &= \sum_{k=1}^{2^d}  W_k^{(2)} ReLU(W_{k, obs(m)}^{(X)} x_{obs(m)} + W_{k, mis(m)}^{(X)} G_{k, mis(m)} + b_k^{(1)}) + b^{(2)}
\end{align*}
We will now show that there exists a configuration of the weights
$W^{(X)}$, $G$, $W^{(2)}$, $b^{(1)}$ and $b^{(2)}$ such that the
predictor $y$ is exactly the Bayes predictor. To do this, we will first
show that we can choose $G$ and $b^{(1)}$ such that the points with a
given missing-values pattern all activate one single hidden unit, and
conversely, a hidden unit can only be activated by a single
missing-values pattern. This setting amounts to having one linear
regression per missing-values pattern. Then, we will show that $W^{(X)}$
and $W^{(2)}$ can be chosen so that for each missing-values pattern, the slope and bias match those of the Bayes predictor.\\

\paragraph{One to one correspondence between missing-values pattern and hidden unit}
In this part, $W^{(X)}$, $W^{(2)}$ and $b^{(2)}$ are considered to be fixed to arbitrary values. We denote by $m_k$, $k\in \llbracket 1, 2^d\rrbracket$, the possible values taken by the mask vector $M$. There is a one-to-one correspondence between missing-values pattern and hidden unit if $G$ and $b^{(1)}$ satisfy the following system of $2^{2d}$ inequations:

\begin{numcases}{\forall x \in \supp(X),\, \forall k \in \llbracket 1, 2^d\rrbracket,}
    W_{k, obs(m_k)}^{(X)} x_{obs(m_k)} + W_{k, mis(m_k)}^{(X)} G_{k, mis(m_k)} + b_{k}^{(1)} \geq 0 & \label{eq:activation}\\
    W_{k, obs(m^\prime)}^{(X)} x_{obs(m^\prime)} + W_{k, mis(m^\prime)}^{(X)} G_{k, mis(m^\prime)} + b_k^{(1)} \leq 0 & $\forall m^\prime \ne m_k$ \label{eq:inactivation}
\end{numcases}
i.e., missing-values pattern $m_k$ activates the $k^{th}$ hidden unit but no other missing-values pattern activates it.\\

Hereafter, we suppose that the support of the data is finite, so that there exist $M \in \mathbb R^+$ such that for any $j\in\llbracket 1, d\rrbracket$, $\left|x_j\right|<M$. As a result, we have:
\begin{align*}
    \abs{W_{k, obs(m_k)}^{(X)} x_{obs(m_k)}} & \leq M \sum_{j \in obs(m_k)} \abs{W_{k, j}^{(X)}}\\
    & \leq M \abs{obs(m_k)} \max_{j \in obs(m_k)} \abs{W_{k, j}^{(X)}}\\
    &= K_k \abs{obs(m_k)}
\end{align*}
where we define $K_k = M \underset{j \in obs(m_k)}{\max} \abs{W_{k, j}^{(X)}}$. We also define $I^{(1)}_k \in \mathbb R$ such that:
\begin{equation}
    \label{eq:I1}
    \forall j \in mis(m_k),\, W_{k, j}^{(X)} G_{k, j} = I^{(1)}_k
\end{equation}
Then satisfying inequation (\ref{eq:activation}) implies satisfying the following inequation:
\begin{equation} \label{eq:activation2}
    \forall k \in \llbracket 1, 2^d\rrbracket,\, - \abs{obs(m_k)}K_k + \abs{mis(m_k)} I^{(1)}_k + b_k^{(1)} \geq 0
\end{equation}

Similarly, we define a quantity $I^{(2)}_k \in \mathbb R$ which satisfies:
\begin{equation}
    \label{eq:I2}
    \forall j \in obs(m_k),\, W_{k, j}^{(X)} G_{k, j} = I^{(2)}_k
\end{equation}
A missing-values pattern $m^\prime \ne m_k$ differs from $m_k$ by a set of entries $\mathcal J \subseteq mis(m_k)$ which are missing in $m_k$ but observed in $m^\prime$, and a set of entries  $\mathcal L \subseteq obs(m_k)$ which are observed in $m_k$ but missing in $m^\prime$. We will call a pair $\Jcal \subseteq mis(m_k)$, $\Lcal \subseteq obs(m_k)$ such that $\abs{\Jcal \cup \Lcal} \ne 0$ a \emph{feasible} pair. With these quantities, satisfying inequation \ref{eq:inactivation} implies satisfying the following inequation:
\begin{equation} \label{eq:inactivation2}
    \forall k \in \llbracket 1, 2^d\rrbracket,\,\forall  (\Jcal, \Lcal) \text{ feasible},\, \left(\abs{obs(m_k)} + \abs \Jcal - \abs \Lcal \right) K_k + \left(\abs{mis(m_k)} - \abs \Jcal \right) I^{(1)}_k + \abs \Lcal I^{(2)}_k + b_k^{(1)} \leq 0
\end{equation}

Thus, by (\ref{eq:activation2}) and (\ref{eq:inactivation2}), a one to one
correspondence between missing-values pattern and hidden unit is possible if there exists $I^{(1)}_k$, $I^{(2)}_k$, $b^{(1)}_k$ such that:
\begin{equation}
\label{eq:b}
\forall k \in \llbracket 1, 2^d\rrbracket,
    \begin{cases}
    \abs{mis(m_k)} I^{(1)}_k + b_k^{(1)} \geq \abs{obs(m_k)}K_k &\\
    \abs{mis(m_k)} I^{(1)}_k + b_k^{(1)} \leq - \abs{obs(m_k)}K_k - (\abs \Jcal - \abs \Lcal)K_k + \abs \Jcal I_k^{(1)} - \abs \Lcal I^{(2)}_k
    & \forall  (\Jcal, \Lcal) \text{ feasible}
\end{cases}
\end{equation}

Because $b_k^{(1)}$ can be any value, this set of inequations admits a solution if for any feasible $(\Jcal, \Lcal)$:
\begin{align*}
    &\abs{obs(m_k)}K_k < - \abs{obs(m_k)}K_k - (\abs \Jcal - \abs \Lcal)K_k + \abs \Jcal I_k^{(1)} - \abs \Lcal I^{(2)}_k\\
    \iff & 2\abs{obs(m_k)}K_k + (\abs \Jcal - \abs \Lcal)K_k < \abs \Jcal I_k^{(1)} - \abs \Lcal I^{(2)}_k\\
    \iff &
    \begin{cases}
        \frac{-2\abs{obs(m_k}K_k}{\abs \Lcal} + K_k > I_k^{(2)} \quad \text{ if } \abs \Jcal = 0\\
        \frac{2\abs{obs(m_k}K_k}{\abs \Jcal} + K_k < I^{(1)}_k \quad \text{ if } \abs \Lcal = 0\\
        I_k^{(1)} > K_k + \frac{\abs{obs(m_k)} K_k}{\abs \Jcal} \quad \text{and} \quad I^{(2)}_k < K_k - \frac{\abs{obs(m_k)} K_k}{\abs \Lcal} \quad \text{otherwise}
    \end{cases}
\end{align*}
Satisfying these inequalities for any feasible $(\Jcal, \Lcal)$ can be achieved by choosing:
\begin{align}
    I_k^{(1)} > (1 + 2\abs{obs(m_k})K_k \label{eq:condition_I1}\\
    I^{(2)}_k < (1 - 2\abs{obs(m_k})K_k \label{eq:condition_I2}
\end{align}
To conclude, it is possible to achieve a one to one correspondence
between missing-values pattern and hidden unit by choosing $G$ and $b^{(1)}$ such that for the $k^{th}$ hidden unit:
\begin{equation}
    \label{eq:one_to_one}
    \begin{cases}
    I_k^{(1)} > (1 + 2\abs{obs(m_k})K_k & \quad \text{by } \ref{eq:I1} \text{ and } \ref{eq:condition_I1}\\
    I_k^{(2)} < (1 - 2\abs{obs(m_k})K_k & \quad \text{by } \ref{eq:I2} \text{ and } \ref{eq:condition_I2}\\
    b_k^{(1)} \text{ satisfies } \ref{eq:b}
\end{cases}
\end{equation}

\paragraph{Equating slopes and biases with that of the Bayes predictor}
We just showed that it is possible to choose $G$ and $b^{(1)}$ such that
the points with a given missing-values pattern all activate one single
hidden unit, and conversely, a hidden unit can only be activated by a
single missing-values pattern. As a consequence, the predictor for an input $(x, m_k) \in \mathbb R^d \times \left\{0, 1 \right\}^d$ is given by:
\begin{align*}
    y(x, m_k) &= \sum_{h=1}^{2^d}  W_h^{(2)} ReLU(W_{h, obs(m_k)}^{(X)} x_{obs(m_k)} + W_{h, mis(m_k)}^{(X)} G_{h, mis(m_k)} + b_h^{(1)}) + b^{(2)}\\
    &= W_k^{(2)} \left(W^{(X)}_{k, obs(m_k)} x_{obs(m_k)} + W_{k, mis(m_k)}^{(X)} G_{k, mis(m_k)} + b_k^{(1)}\right) + b^{(2)}
\end{align*}
For each missing-values pattern, it is now easy to choose $W^{(X)}_{k,
obs(m_k)}$ and $W^{(2)}$ so that the slopes and biases of this linear
function match those of the Bayes predictor defined in
proposition~\ref{prop:condexp}. Let $\beta_k \in \mathbb
R^{\abs{obs(m_k)}}$ and $\alpha_k \in \mathbb R$ be the slope and bias of
the Bayes predictor for missing-values pattern $m_k$. Then setting
\begin{align}
    \label{eq:slopes_and_biases}
    \begin{cases}
        W_k^{(2)} \left(W_{k, mis(m_k)}^{(X)} G_{k, mis(m_k)} + b_k^{(1)} \right) + b^{(2)} = \alpha_k\\
        W_k^{(2)} W_{k, obs(m_k)}^{(X)} = \beta_k
    \end{cases}
\end{align}
equates the slope and bias of the MLP to those of the bias predictor.

\paragraph{Construction of weights for which the MLP is the Bayes predictor} We have shown that achieving a one to one correspondence between missing data pattern and hidden units involves satisfying a set of inequations on the weights (\ref{eq:one_to_one}), while equating the slopes and biases to those of the Bayes predictor involves another set of equations (\ref{eq:slopes_and_biases}). To terminate the proof, it remains to be shown that the whole system of equations and inequations admits a solution.\\

We start by working on the one-to-one correspondence system of inequations (\ref{eq:one_to_one}). Let $\epsilon > 0$ be a parameter. Inequations (\ref{eq:condition_I1}) and (\ref{eq:condition_I2}) are satisfied by choosing:
\begin{align}
    I_k^{(1)} = (1 + 2\abs{obs(m_k})K_k + \epsilon \label{eq:condition_I1_eps}\\
    I^{(2)}_k = (1 - 2\abs{obs(m_k})K_k - \epsilon \label{eq:condition_I2_eps}
\end{align}
According to the second inequation in (\ref{eq:b}), $b_k^{(1)}$ is upper bounded as:
\begin{align*}
    b_k^{(1)} \leq - \abs{obs(m_k)}K_k - \abs{mis(m_k)} I^{(1)}_k - (\abs \Jcal - \abs \Lcal)K_k + \abs \Jcal I_k^{(1)} - \abs \Lcal I^{(2)}_k
\end{align*}
This inequation can be simplified:
\begin{align*}
     b_k^{(1)} &\leq - \abs{obs(m_k)}K_k - \abs{mis(m_k)} I^{(1)}_k + \abs \Jcal (I_k^{(1)} - K_k) + \abs \Lcal (K_k - I^{(2)}_k)\\
    &= -\abs{obs(m_k)}K_k - \abs{mis(m_k)} + \abs \Jcal (2\abs{obs(m_k)}K_k + \epsilon) + \abs \Lcal(2\abs{obs(m_k)} K_k +\epsilon)
\end{align*}
The smallest upper bound is obtained for $\abs{\Jcal \cup \Lcal}=1$ which gives:
\begin{equation*}
     b_k^{(1)} \leq \abs{obs(m_k)}K_k - \abs{mis(m_k)} +  \epsilon
\end{equation*}
According to the first inequation in (\ref{eq:b}), $b_k^{(1)}$ is also lower bounded as:
\begin{equation*}
    b_k^{(1)} \geq \abs{obs(m_k)}K_k - \abs{mis(m_k)} I^{(1)}_k
\end{equation*}
A valid choice for $b_k^{(1)}$ is the mean of its upper and lower bounds. We therefore choose to set:
\begin{equation}
    b_k^{(1)} = \abs{obs(m_k)}K_k - \abs{mis(m_k)} I^{(1)}_k + \frac{\epsilon}{2}
\end{equation}

To summarise, we can restate the conditions for one to one correspondence as:
\begin{align}
    \label{eq:one_to_one_eps}
    \epsilon &> 0\\
    \label{eq:one_to_one_I1}
    I_k^{(1)} &= (1 + 2\abs{obs(m_k})K_k + \epsilon\\
    \label{eq:one_to_one_I2}
    I^{(2)}_k &= (1 - 2\abs{obs(m_k})K_k - \epsilon\\
    \label{eq:one_to_one_b}
    b_k^{(1)} &= \abs{obs(m_k)}K_k - \abs{mis(m_k)} I^{(1)}_k + \frac{\epsilon}{2}
\end{align}

We now turn to the slopes and biases equations (\ref{eq:slopes_and_biases}). Replacing $b_k^{(1)}$ in the bias equation by its value in (\ref{eq:one_to_one_b}) we get:
\begin{align*}
    & W_k^{(2)} \left(W_{k, mis(m_k)}^{(X)} G_{k, mis(m_k)} + b_k^{(1)} \right) + b^{(2)} = \alpha_k \\
    \iff & W_k^{(2)} \left(\abs{mis(m_k)} I^{(1)}_k + b_k^{(1)} \right) + b^{(2)} = \alpha_k \\
    \iff & W_k^{(2)} \left(\abs{obs(m_k)}K_k + \frac{\epsilon}{2} \right) + b^{(2)} = \alpha_k \\
\end{align*}
Putting together the one to one correspondence conditions, the slope and biases equations as well as the variable definitions, we get a set of 8 equations and 1 inequation:
\begin{align}
    \label{eq:all_eps}
    &\epsilon > 0\\
    \label{eq:all_I1}
    &I_k^{(1)} = (1 + 2\abs{obs(m_k})K_k + \epsilon\\
    \label{eq:all_I2}
    &I^{(2)}_k = (1 - 2\abs{obs(m_k})K_k - \epsilon\\
    \label{eq:all_b}
    &b_k^{(1)} = \abs{obs(m_k)}K_k - \abs{mis(m_k)} I^{(1)}_k + \frac{\epsilon}{2}\\
    \label{eq:all_bias}
    &W_k^{(2)} \left(\abs{obs(m_k)}K_k + \frac{\epsilon}{2} \right) + b^{(2)} = \alpha_k\\
    \label{eq:all_slope}
    &W_k^{(2)} W_{k, obs(m_k)}^{(X)} = \beta_k\\
    \label{eq:all_K}
    &K_k = M \underset{j \in obs(m_k)}{\max} \abs{W_{k, j}^{(X)}}\\
    \label{eq:all_Gmis}
    &\forall j \in mis(m_k),\, W_{k, j}^{(X)} G_{k, j} = I^{(1)}_k\\
    \label{eq:all_Gobs}
    &\forall j \in obs(m_k),\, W_{k, j}^{(X)} G_{k, j} = I^{(2)}_k
\end{align}
One can verify that this system of inequations has a solution. Indeed, choose $W_{k, obs(m_k)}^{(X)}$ proportional to $\beta_k$ so that equation~(\ref{eq:all_slope}) can be verified. This imposes a value for $W_k^{(2)}$ via (\ref{eq:all_slope}) and a value for $K_k$ via (\ref{eq:all_K}). In turn, it imposes a value for $\epsilon$ via (\ref{eq:all_bias}): $\epsilon = 2\left( \alpha_k - b^{(2)} - W_k^{(2)} \abs{obs(m_k)}K_k \right)$. The value obtained for $\epsilon$ is positive if we choose $b^{(2)}$ sufficiently negative. Note that there is one single value of $b^{(2)}$ for all units so $b^{(2)}$ should be chosen by considering all units. Then $K_k$ and $\epsilon$ impose $I^{(1)}_k$ and $I^{(2)}_k$ via (\ref{eq:all_I1}) and (\ref{eq:all_I2}). $K_k$, $\epsilon$ and $I^{(1)}_k$ impose $b_k^{(1)}$ via (\ref{eq:all_b}). Finally $W_{k, .}^{(X)}$, $I^{(1)}_k$ and $I^{(2)}_k$ impose $G$ via (\ref{eq:all_Gmis}) and (\ref{eq:all_Gobs}).

\textbf{Case 2}: Suppose that $\exists k \in \llbracket 1, 2^d \rrbracket,\, \exists j \in \llbracket 1, d \rrbracket: W_{k,j}^{(X)} = 0$.\\

Recall that the proof which shows that we can achieve a one to one
correspondence between missing-values pattern and hidden unit relies on the assumption that $\forall k \in \llbracket 1, 2^d \rrbracket,\, \forall j \in \llbracket 1, d \rrbracket,\, W_{k,j}^{(X)} \ne 0$. However, if there is a slope $\beta_k$ of the Bayes predictor such that its $j^{th}$ coefficient is 0, then we must choose $W_{k, j}^{(X)} = 0$ to achieve Bayes consistency. In such a case, we need to extend the one to one correspondence proof to the case where an entry of $W_{k, j}^{(X)}$ can be zero. It turns out to be easy.

In this case, we cannot pose $G_{k, j} = W_{k, j}^{(M)}/W_{k, j}^{(X)}$. Let $\mathcal Z_k \subseteq \llbracket 1, d \rrbracket$ be the set of indices such that $\forall j \in \mathcal Z_k,\, W_{k,j}^{(X)} = 0$. The whole reasoning exposed in case 1 still holds if we replace $obs(m)$ by $obs(m) \setminus \mathcal Z_k$ and $mis(m)$ by $mis(m) \setminus \mathcal Z_k$.

\section{Complementary figures}
\subsection{Comparison at $n=75\,000$}

\autoref{fig:boxplots} gives a box plot view of the behavior at
$n=75\,000$. It is complementary to the learning curves, though it
carries the same information.

\begin{figure*}[t!]
    \includegraphics[width=.39\linewidth]{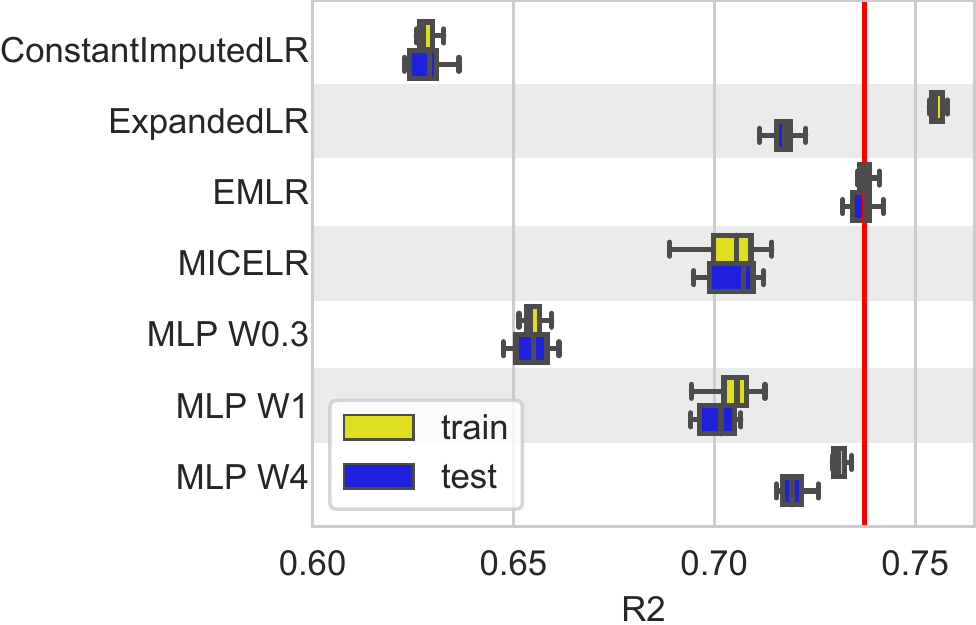}%
    \llap{\raisebox{.252\linewidth}{\sffamily\small{\bfseries Mixture 1}
	    (MCAR)\quad\qquad}}%
    \hspace*{-.095\linewidth}\hfill%
    \includegraphics[width=.39\linewidth]{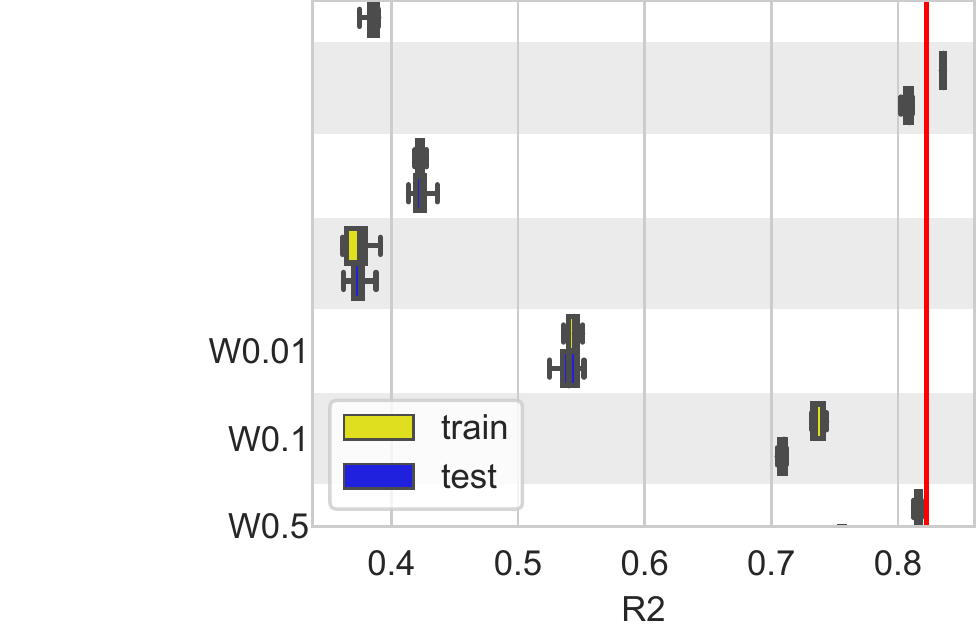}%
    \llap{\raisebox{.252\linewidth}{\sffamily\bfseries\small Mixture 3
	    \qquad\qquad}}%
    \hspace*{-.1\linewidth}\hfill%
    \includegraphics[width=.39\linewidth]{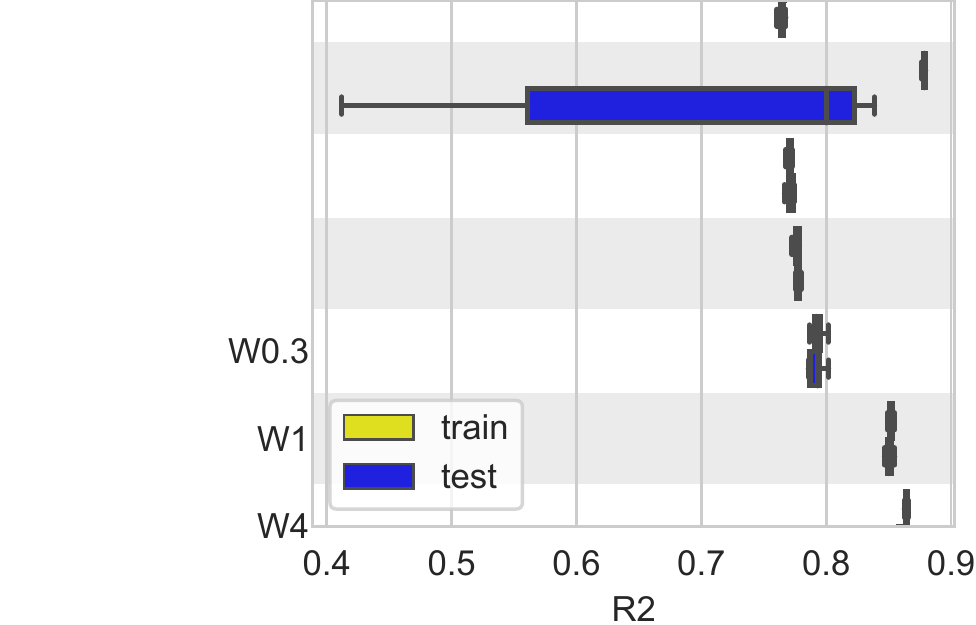}%
    \llap{\raisebox{.252\linewidth}{\sffamily\small{\bfseries Self-masked}
	(MNAR)~\qquad}}%
    \caption{\textbf{Prediction accuracy} R2 score for the 3 data types with $n=75,000$ training samples and in dimension $d=10$. The quantities displayed are the mean and standard deviation over 5 repetitions.
    }
    \label{fig:boxplots}
\end{figure*}

\subsection{Experiments on growing MLP's width}
\label{sec:scaling_q}%

\autoref{fig:MLP_scaling_q} shows the performance of the MLP in the
various simulation scenarios as a function of the number of hidden units
of the networks. In each scenario, the number of hidden units is taken
proportional to a function of the input dimension $d$:
\begin{description}
    \item[mixture 1]: $n_h \propto d$
    \item[mixture 3]: $n_h \propto 2^d$
    \item[selfmasked]: $n_h \propto d$
\end{description}
These results show that the number of hidden units needed by the MLP to
predict well are a function of the complexity of the underlying
data-generating mechanism. Indeed, for the \emph{mixture 1}, the MLP only
needs  $n_h \propto d$ while the missing
values are MCAR, and therefore ignorable. For \emph{selfmasked}, the
challenge is to find the right set of thresholds, after which the
prediction is relatively simple: the MLP also needs 
$n_h \propto d$. On the opposite, for \emph{mixture 3}, the
multiple Gaussians create couplings in the prediction function; as the
consequence, the MLP needs $n_h \propto 2^d$.

\begin{figure*}[t!]
    \includegraphics[width=.33\linewidth]{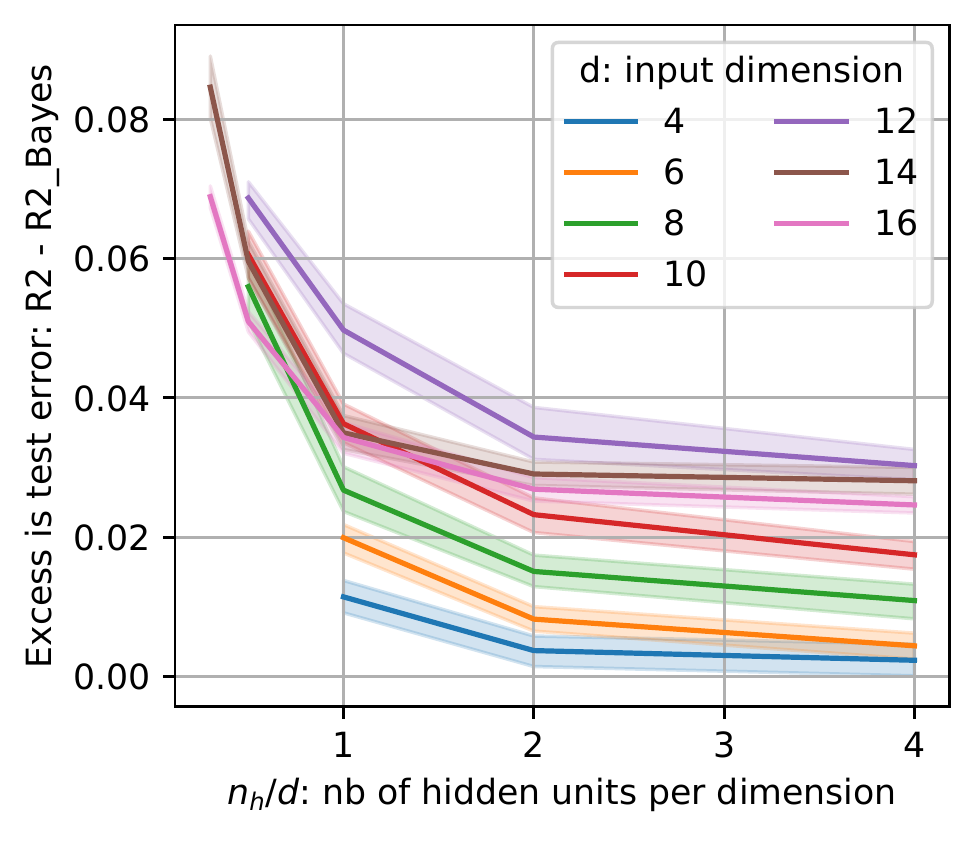}%
    \llap{\raisebox{.285\linewidth}{\sffamily\bfseries\small Mixture 1
	    \qquad\qquad}}%
    \hfill%
    \includegraphics[width=.33\linewidth]{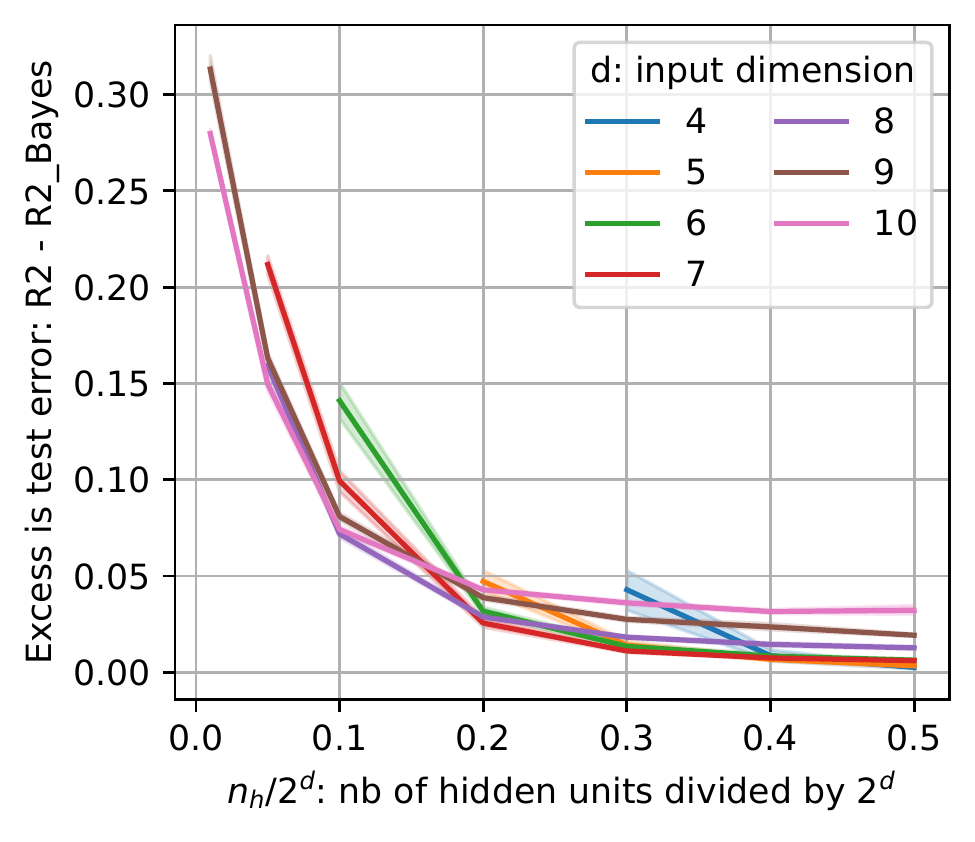}%
    \llap{\raisebox{.285\linewidth}{\sffamily\bfseries\small Mixture 3
	    \qquad\qquad}}%
    \hfill%
    \includegraphics[width=.33\linewidth]{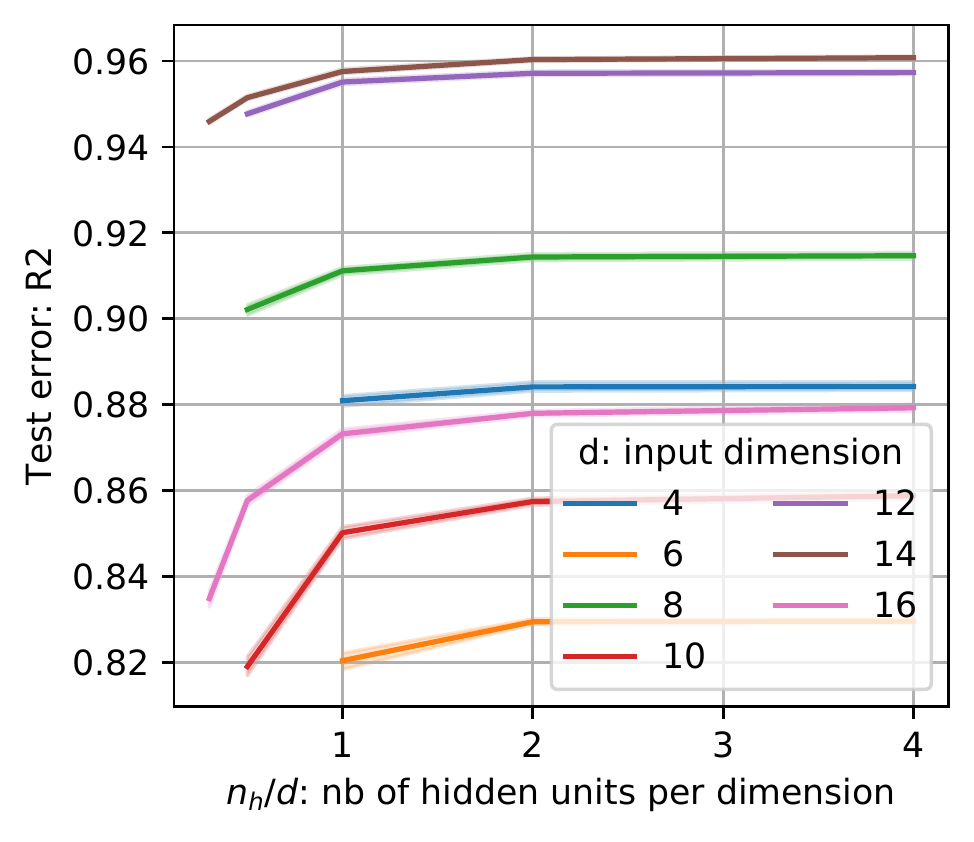}%
    \llap{\raisebox{.285\linewidth}{\sffamily\bfseries\small Self-masked
	\qquad\qquad}}%
    \caption{\textbf{Performance of the one hidden layer MLP as a
function of its number of hidden units} For the mixtures of Gaussians,
the performance is given as the difference between the R2 score of the
MLP and that of the Bayes predictor. For each dimension $d$, multiple
MLPs are trained, each with a different number of hidden units given by
$q \times d$ for mixture 1 and self-masked, $q \times 2^d$ for mixture 3. 75,000 training samples were used for Mixture 1 and Self-masked and 375,000 for Mixture 3.
    \label{fig:MLP_scaling_q}%
    }
\end{figure*}

\end{document}